\documentclass{article}
\usepackage{float}
\usepackage{graphicx}
\usepackage{amssymb, amsmath, amsthm,mathtools}
\usepackage{epstopdf}

\usepackage{bbm}
\usepackage[margin=1in]{geometry}
\usepackage{layout, color}
\usepackage{enumerate}
\usepackage{enumitem}
\usepackage{authblk}
\usepackage{algorithm}
\usepackage{algpseudocode}
\usepackage{setspace}
\usepackage{hyperref}
\usepackage{wrapfig}
\usepackage[compress]{natbib}
\usepackage{caption} 
\newsavebox{\algobox}

\usepackage{url}
\usepackage{subfig}
\usepackage{mdwlist}
\usepackage{multirow}
\usepackage{amsthm}
\usepackage{color}
\usepackage[dvipsnames]{xcolor}

\usepackage[title]{appendix}

\newtheorem{theorem}{Theorem}[section]

\newtheorem{proposition}[theorem]{Proposition}
\newtheorem{assump}{Assumption}

\newtheorem{lemma}[theorem]{Lemma}
\newtheorem{definition}[theorem]{Definition}

\theoremstyle{definition}

\newtheorem{conjecture*}{Conjecture}
\theoremstyle{plain}

\newcommand{\Alpha}{\boldsymbol{\alpha}}

\newcommand{\Beta}{\boldsymbol{\beta}}

\usepackage[normalem]{ulem}

\newcommand{\Zx}{Z_x^{(\theta)}}
\newcommand{\Zy}{Z_y^{(\theta)}}

\newenvironment{keywords}{%
    \small\textbf{Keywords:}\space
}{%
    \par
}

\title{Density-Reweighted Entropic Optimal Transport: Decoupling Geometry from Sampling Density}
\author{ Keyi Li${^{1,*}}$~~~~Yuval Kluger${^{1,2}}$~~~~Boris Landa${^{1,3,*}}$\\
\small{${^1}$Program of Applied \& Computational Mathematics, Yale University}\\
\small{${^2}$Department of Pathology, Yale University}\\
\small{${^3}$Department of Electrical and Computer Engineering, Yale University}\\
\small{${^*}$Corresponding author. Email: keyi.y.li@yale.edu, boris.landa@yale.edu}
}

\date{}
\begin{document}
\maketitle

\begin{abstract}
    Dataset alignment is a central step in data analysis across science and engineering, where the goal is to match observations between datasets. Entropic Optimal Transport (EOT) offers a computationally tractable framework for this task by encoding cross-dataset affinities in a transport plan. However, when two datasets are sampled from geometrically similar low-dimensional structures with substantially different sampling densities, the EOT plan may match points by relative sampling density rather than geometric proximity, yielding geometrically misleading correspondences. To address this issue, we propose a density-reweighted EOT framework in which the influence of sampling density on the transport plan can be discounted to a desired degree, ranging from standard EOT to alignment driven purely by underlying geometry. Under suitable regularity conditions, we establish convergence of the reweighted EOT plan to a family of population-level plans whose dependence on sampling density is made explicit. Through simulations, we show that our approach recovers geometrically faithful correspondences, improving over related EOT-based frameworks when datasets exhibit substantial sampling density disparity. 
\end{abstract}



\section{Introduction}

    Dataset alignment, the problem of establishing correspondences between observations across multiple datasets, is a fundamental task arising throughout the scientific and engineering disciplines~\cite{ align_ex2, align_ex4, align_ex5}. In genomics, for instance, datasets collected under heterogeneous conditions, such as different experimental batches or measurement technologies, often exhibit similar low-dimensional structures characterizing the underlying biological process, yet undergo systematic deformations that render direct comparison infeasible~\cite{genomic_batch, genomics_batch2, genomics_batch3}. Establishing correspondences in a geometrically meaningful way, i.e., identifying for each data point its most similar counterparts in the other dataset with respect to their underlying geometry, is therefore essential for reliable downstream analysis~\cite{align_ex2}. 

    Optimal Transport (OT)~\cite{OT_intro} provides a natural mathematical formalization of this task by seeking a transport plan that minimizes the total cost of moving mass between two datasets, where the cost function encodes geometric proximity between data points. However, exact OT scales cubically in the number of samples~\cite{OT_intro2}, limiting its applicability to large-scale datasets. Entropic Optimal Transport (EOT)~\cite{EOT}, an entropy-regularized convex relaxation of OT, offers a computationally tractable alternative. Formally, given two finite point clouds $\mathbf{X} = \{\mathbf{x}_i\}_{i=1}^m$ and $\mathbf{Y} = \{\mathbf{y}_j\}_{j=1}^n$, a cost matrix $\mathbf{C} \in \mathbb{R}^{m \times n}$ with entries $C_{ij} = c(\mathbf{x}_i, \mathbf{y}_j)$ denoting the transport cost per unit mass from $\mathbf{x}_i$ to $\mathbf{y}_j$, and marginal constraints $\mathbf{a} \in \mathbb{R}_+^m$, $\mathbf{b} \in \mathbb{R}_+^n$ satisfying $\mathbf{1}^\top \mathbf{a} = \mathbf{1}^\top \mathbf{b}$, EOT seeks the transport plan $\mathbf{W}^*$ (also referred to as EOT plan hereafter) solving
    \begin{equation}\label{eq:EOT_problem}
        \begin{aligned}
        \min_{\tilde{\mathbf{W}} \in \mathbb{R}_{+}^{m\times n}} \quad
        & \sum_{i=1}^m \sum_{j=1}^n C_{ij}\, \tilde{W}_{ij}
        \;+\; \varepsilon \sum_{i=1}^m \sum_{j=1}^n \tilde{W}_{ij} \log \tilde{W}_{ij} \\
        \text{s.t.} \quad
        & \sum_{j=1}^n \tilde{W}_{ij} = a_i, \quad \sum_{i=1}^m \tilde{W}_{ij} = b_j,
        \quad \forall\, i \in [m],\, \forall\, j \in [n],
        \end{aligned}
    \end{equation}
    where $\varepsilon > 0$ is the entropic regularization parameter controlling the
    diffuseness of the resulting transport plan. The EOT problem in~\eqref{eq:EOT_problem}
    is strictly convex and admits a unique solution of the form
    \begin{equation}\label{eq:EOT_standard_solution_form}
        W^*_{ij} \;=\; u_i\, K_{ij}\, v_j, \qquad
        K_{ij} \;=\; \exp\!\left(-C_{ij}/\varepsilon\right),
        \quad \forall\, i \in [m],\, \forall\, j \in [n],
    \end{equation}
    where $\mathbf{u} \in \mathbb{R}^m_{+}$ and $\mathbf{v} \in \mathbb{R}^n_{+}$ are
    positive scaling vectors that can be computed efficiently via the Sinkhorn
    algorithm~\cite{sinkhorn}. The EOT plan $\mathbf{W}^*$ 
    reflects pairwise affinities between data points across the two datasets, and thus provides a natural mechanism for establishing correspondence. Beyond its computational efficiency~\cite{computational_sinkhorn, computational_sinkhorn2}, EOT enjoys favorable statistical properties~\cite{statistical_sinkhorn}, such as robustness to perturbations in the input datasets, making it well-suited for large-scale modern datasets. It has found successful application in domain adaptation~\cite{EOT_domain_adaptation, EOT_domain2}, generative modeling~\cite{generative1, generative2}, computer vision~\cite{cv1, cv2}, and computational biology~\cite{waddington_OT}, among others~\cite{OT_intro2}.

    In dataset alignment tasks, the EOT problem~\eqref{eq:EOT_problem} is typically applied with squared Euclidean distances as the cost matrix and uniform empirical marginals reflecting the respective dataset sizes. Concretely, one sets $C_{ij} = \|\mathbf{x}_i - \mathbf{y}_j\|^2$, $a_i = n$ for all $i \in [m]$, and $b_j = m$ for all $j \in [n]$. However, this formulation, combined with the inherent mass-conservation requirement of EOT, can lead to biased alignment when datasets are class-imbalanced or, more generally, when they exhibit sampling density mismatch~\cite{problem_eot,problem_eot2}. Specifically, when two datasets are sampled from geometrically similar low-dimensional structures (e.g., manifolds) but with substantially different sampling densities, the EOT plan may match points across datasets based on their relative sampling densities rather than their geometric proximity, producing correspondences that are geometrically misleading; see Figure~\ref{fig:demo}(a). Such density mismatches arise naturally in practice: in genomics, for instance, cells profiled across experimental conditions, sequencing batches, or measurement technologies may exhibit artificially skewed sampling densities due to technical biases introduced by the acquisition process itself~\cite{genomic_batch, genomics_batch2, genomics_batch3}. While Unbalanced Optimal Transport (UOT)~\cite{UOT} (discussed in Section~\ref{sec:related_work}) mitigates this limitation by allowing variation in total mass, it requires additional hyperparameter tuning and offers little intuition on how the resulting transport plan is jointly influenced by sampling density and geometric similarity.
    \begin{figure}[t]
        \centering
        \includegraphics[width=0.9\textwidth]{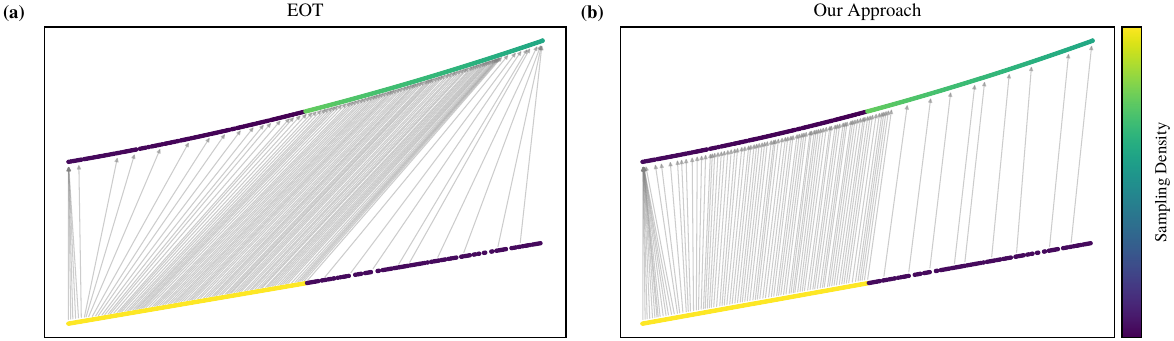}
        \caption{\textbf{Sampling density mismatch biases EOT-based dataset alignment.} Two datasets---a line segment (bottom) and an arc (top)---share similar underlying geometry but differ substantially in sampling density (yellow/green: high density; blue/purple: low density). Gray arrows indicate maximum-coupling correspondences, where each point is matched to its highest-affinity counterpart in the other dataset. \textbf{(a)}~Standard EOT routes correspondences toward high-density regions, producing geometrically misleading matchings. \textbf{(b)}~Our approach recovers correspondences that are driven by the underlying geometry. See Supplement~\ref{sec:demo} for the experiment details.}
        \label{fig:demo}
    \end{figure}
    
    \paragraph{Our Contribution} 
    In this work, we address the challenge of obtaining geometrically coherent EOT-based alignments between datasets sampled from manifolds with similar underlying geometry but substantially different sampling densities. In this setting, standard EOT can produce correspondences that are strongly distorted by density variation and therefore geometrically misleading; see Figure~\ref{fig:demo}(a). Our goal is to recover the alignment that would arise if both manifolds were sampled uniformly, so that correspondences are driven primarily by geometry rather than by sampling density; see Figure~\ref{fig:demo}(b). More broadly, we aim to provide a principled mechanism for interpolating between purely geometry-driven alignment and the density-sensitive alignment produced by standard EOT.

    \textbf{A density-reweighted EOT framework for geometry-driven alignment.} We show that the dependence of the EOT plan on sampling density can be mitigated through density reweighting. Specifically, while standard EOT with uniform empirical marginals amounts to scaling a Gaussian kernel to have constant row and column sums, we propose an approach that instead reweighs the Gaussian kernel by local sampling density and seeks scaling factors that satisfy suitably reweighted marginal constraints. Such density-reweighting assigns greater importance to data points in sparse regions and discounts data points in dense regions, compensating for sampling density variation. The reweighting strength can be controlled by a density discounting factor $\theta \in [0, 1]$. At $\theta = 0$, the method reduces to standard EOT with uniform empirical marginals; at $\theta = 1$, the population limit is independent of the sampling densities, so correspondences are driven by the underlying geometry; intermediate values of $\theta$ allow a trade-off between the two regimes (see Figure~\ref{fig:demo_alpha}). Algorithm~\ref{alg:EOT_invariant} outlines our proposed approach, where $\mathbf{W}^{(\theta)}$ from~\eqref{eq:EOT_solution_form} is the reweighted plan.

    We prove that under suitable regularity conditions on the manifold geometry and sampling densities, $\mathbf{W}^{(\theta)}$ converges with high probability as $m,n\rightarrow\infty$, up to a global scaling constant, to a population-level EOT plan whose dependence on sampling density is explicitly governed by $\theta \in [0, 1]$. In particular, at $\theta = 1$, $\mathbf{W}^{(\theta)}$ converges to the population-level EOT plan under uniform measures on both manifolds.
    This is established in two settings: when the densities are known exactly (Theorem~\ref{thm:finite_population_density_known}) and when the density of each dataset is estimated using a kernel density estimator (Theorem~\ref{thm:finite_population_density_estimated}), with explicit convergence rates.
    Finally, we demonstrate empirically that our approach outperforms standard EOT and UOT across multiple hyperparameter configurations in recovering geometrically faithful correspondences between two datasets with substantial sampling density disparity. In these experiments, we also discuss practical considerations and provide guidance to facilitate deployment of the proposed framework.

    \begin{figure}[!ht]
        \begin{minipage}[t]{0.6\textwidth}
            \vspace{0pt}
            \begin{algorithm}[H]
                \caption{Density-reweighted EOT for Geometry-driven Alignment}
                \label{alg:EOT_invariant}
    
                \textbf{Input:} Datasets $\mathbf{X} = \{\mathbf{x}_i\}_{i=1}^m \subset \mathbb{R}^{D}$ 
                and $\mathbf{Y} = \{\mathbf{y}_j\}_{j=1}^n \subset \mathbb{R}^{D}$; positive density estimates 
                $\hat{f}_i$ and $\hat{g}_j$; EOT regularization $\varepsilon > 0$; 
                discounting factor $\theta \in [0,1]$.
    
                \begin{algorithmic}[1]
                    \State \textbf{Compute scaling constant.}
                    \begin{equation}\label{eq:S_scale}
                        S^{(\theta)} = \frac{\dfrac{1}{m} \sum_{i=1}^m \dfrac{1}{(\hat{f}_i)^{\,\theta}}}
                                 {\dfrac{1}{n}\sum_{j=1}^n \dfrac{1}{(\hat{g}_j)^{\,\theta}}}.
                    \end{equation}
    
                    \State \textbf{Form density-adjusted kernel.}
                    \begin{equation}\label{eq:M}
                        M_{ij}^{(\theta)} = \frac{\exp\!\left(-\|\mathbf{x}_i - \mathbf{y}_j\|_2^2 / \varepsilon\right)}
                                      {(\hat{f}_i)^{\,\theta}\,(\hat{g}_j)^{\,\theta}},
                        \quad \forall i\in [m],\, j \in [n].
                    \end{equation}
    
                    \State \textbf{Sinkhorn scaling.} Find $\boldsymbol{\alpha} \in \mathbb{R}_+^m$,
                    $\boldsymbol{\beta} \in \mathbb{R}_+^n$ such that
                    \begin{equation}\label{eq:row_col_constraint}
                    \begin{aligned}
                        \frac{1}{n}\sum_{j=1}^{n} \alpha_i^{(\theta)} M_{ij}^{(\theta)} \beta_j^{(\theta)} &= \frac{1}{\sqrt{S^{(\theta)}}\,(\hat{f}_i)^{\,\theta}}, \\
                        \frac{1}{m}\sum_{i=1}^{m} \alpha_i^{(\theta)} M_{ij}^{(\theta)} \beta_j^{(\theta)} &= \frac{\sqrt{S^{(\theta)}}}{(\hat{g}_j)^{\,\theta}}.
                    \end{aligned}
                    \end{equation}
    
                    \State \textbf{Recover the density-reweighted EOT} plan.
                    \begin{equation}\label{eq:EOT_solution_form}
                        W_{ij}^{(\theta)} = \alpha_i^{(\theta)} \exp\!\left(-\|\mathbf{x}_i - \mathbf{y}_j\|_2^2 / \varepsilon\right) \beta_j^{(\theta)}.
                    \end{equation}
                \end{algorithmic}
            \end{algorithm}
        \end{minipage}
        \hfill
        \begin{minipage}[t]{0.37\textwidth}
            \vspace{0pt}
            \centering
            \includegraphics[width=0.92\textwidth]{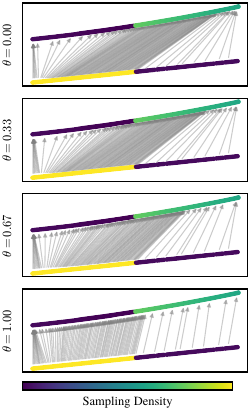}
            \captionof{figure}{\textbf{Discounting the influence of sampling density via 
            $\theta$ in Algorithm~\ref{alg:EOT_invariant}.} Each panel shows the 
            maximum-coupling correspondences for a different value of $\theta$. 
            At $\theta = 0$, the method reduces to standard EOT; at $\theta = 1$, 
            it recovers purely geometry-based alignment.}
            \label{fig:demo_alpha}
        \end{minipage}
    \end{figure}    

\section{Related Work}\label{sec:related_work}
    \paragraph{Unbalanced Optimal Transport}

    Unbalanced Optimal Transport (UOT)~\cite{UOT} extends EOT by relaxing the hard marginal constraints in~\eqref{eq:EOT_problem} through soft penalty terms, allowing the transport plan to deviate from the prescribed marginals and thus accommodate variation in the total transported mass. A popular choice of penalty is the class of $f$-divergences~\cite{UOT, UOT_TV}, among which the KL-divergence formulation is widely adopted. Formally, UOT with the KL-divergence penalty seeks a transport plan solving
    \begin{equation}\label{eq:UOT_problem}
    \begin{aligned}
        \min_{\tilde{\mathbf{W}} \in \mathbb{R}_{+}^{m\times n}} \quad
        & \sum_{i=1}^m \sum_{j=1}^n C_{ij}\, \tilde{W}_{ij}
        \;+\; \varepsilon \sum_{i=1}^m \sum_{j=1}^n \tilde{W}_{ij} \log \tilde{W}_{ij}
        +\; \lambda_1\, \mathrm{KL}\!\left(\tilde{\mathbf{W}}\mathbf{1}_n \,\Big\|\, \mathbf{a}\right)
        \;+\; \lambda_2\, \mathrm{KL}\!\left(\tilde{\mathbf{W}}^\top\mathbf{1}_m \,\Big\|\, \mathbf{b}\right),
    \end{aligned}
    \end{equation}
    where $\mathrm{KL}(\mathbf{p} \| \mathbf{q}) = \sum_i \left( p_i \log(p_i / q_i) - p_i + q_i\right)$ is the generalized KL divergence for non-negative vectors $\mathbf{p}, \mathbf{q}$ of the same dimension, and $\lambda_1, \lambda_2 > 0$ control the degree of marginal relaxation. As $\lambda_1, \lambda_2 \to \infty$, the marginal constraints become increasingly strict and the UOT plan recovers the EOT plan; as $\lambda_1, \lambda_2 \to 0$, the marginal penalties vanish and the UOT plan reduces to the Gaussian kernel, which encodes cross-data similarities via raw Euclidean distances. UOT has been studied from multiple perspectives; we refer the reader to~\cite{UOT,OT_intro, OT_intro2} for  comprehensive overviews, to~\cite{UOT_theory} for its statistical properties, to~\cite{UOT_TV, UOT_alg} for computational algorithms, and to~\cite{waddington_OT, UOT_app2} for applications in machine learning. 

    In the context of dataset alignment, while relaxing strict mass conservation
    allows the UOT plan to suffer less from dataset-specific variation in sampling density, the influence of $(\lambda_1, \lambda_2)$ on the
    resulting plan is implicit: it is unclear how the choice of
    $(\lambda_1, \lambda_2)$ controls the extent to which $\tilde{\mathbf{W}}$
    reflects the underlying geometry versus the sampling densities, and to our knowledge this interplay has not been characterized in prior work. Our approach addresses this limitation by making the density-geometry tradeoff explicit and easily interpretable via the discounting parameter $\theta$.

    \paragraph{Manifold Learning} Our approach is conceptually motivated by Diffusion Maps~\cite{diffusion_maps}, which
    provides a framework for manifold learning with the ability to suppress the influence of the sampling density on the recovered geometry. A key observation underlying
    Diffusion Maps is that the graph Laplacian constructed from a kernel $k(x,y) =
    \exp(-\|x-y\|^2/h^2)$ depends on the sampling density of the data, mixing
    geometric structure with the distribution of sample points. To mitigate this,
    Diffusion Maps introduce the following $\alpha$-normalization of the kernel,
    \begin{equation}
        k_\alpha(x, y) = \frac{k(x,y)}{d(x)^\alpha \, d(y)^\alpha}, \quad
        d(x) = \int_X k(x,y)\,d\mu(y),
    \end{equation}
    where $d(x)$ is a local measure of volume, and $\alpha
    \in [0,1]$ controls the degree to which sampling density is discounted. As $h \to 0$, the infinitesimal generator of the Markov chain induced by $k_\alpha$ converges to a differential operator whose dependence on the sampling density is explicitly controlled by $\alpha$: at $\alpha = 0$, the differential operator retains full dependence on the sampling density, recovering the random-walk graph Laplacian; at $\alpha = 1$, the density dependence is completely removed and the differential operator recovers the Laplace--Beltrami operator on the manifold, independently of the distribution of sample points. We extend this idea to the EOT framework by incorporating an analogous density discounting via a parameter $\theta \in [0,1]$ (Algorithm~\ref{alg:EOT_invariant}). At $\theta = 0$, the standard EOT plan is recovered; at $\theta = 1$, the resulting transport plan is invariant to the sampling measure, with correspondences governed purely by geometry; and intermediate values of $\theta$ provide a principled balance between the two regimes.

    Our work is also closely related to EOT Eigenmaps~\cite{boris_EOT}, which established a connection between the EOT plan and manifold learning. Specifically, \cite{boris_EOT} showed that the singular vectors of the EOT plan can be used to jointly embed two datasets sampled from a shared low-dimensional manifold, possibly subject to dataset-specific deformations and corruptions. A central result is that, when viewed as a linear operator, the EOT plan converges in the large-sample limit to a kernel operator whose singular vectors encode the intrinsic geometry of the underlying manifold. This analysis assumes that the shared manifold is sampled according to the same density in both datasets. In contrast, our work addresses the setting in which two datasets have similar underlying geometry while exhibiting different sampling densities. Therefore, our proposed reweighting can enable joint manifold learning and embedding methods that are robust to dataset-specific variation in sampling density.

\section{Density-Reweighted EOT and Theoretical Guarantees}

    Let $\mathbf{X} = \{\mathbf{x}_i\}_{i=1}^{m} \subset \mathbb{R}^D$ and $\mathbf{Y} = \{\mathbf{y}_j\}_{j=1}^{n} \subset \mathbb{R}^D$ be two point clouds sampled i.i.d.\ from probability measures $\mu\, d\omega$ and $\nu\, d\eta$, respectively, where $d\omega$ and $d\eta$ denote the volume forms on Riemannian manifolds $\mathcal{M}_x, \mathcal{M}_y \subset \mathbb{R}^D$ with intrinsic dimensions $d_1, d_2$, and $\mu$, $\nu$ are the corresponding sampling density functions. We begin with the following proposition, which motivates Algorithm~\ref{alg:EOT_invariant} by connecting it to a suitably modified variant of the optimization problem~\eqref{eq:EOT_problem}.
    
    \begin{proposition}\label{prop:density_invariant_eot}
        The reweighted EOT plan $\mathbf{W}^{(\theta)}$ in~\eqref{eq:EOT_solution_form} is the unique minimizer of the following optimization problem:
        \begin{equation}\label{eq:density_invariant_opt}
            \begin{aligned}
            \operatorname*{min}_{_{\tilde{\mathbf{W}} \in \mathbb{R}_{+}^{m\times n}}}& \sum_i^m \sum_j^n \frac{1}{(\hat{f}_i)^{\,\theta}\,(\hat{g}_j)^{\,\theta}} \cdot \|\mathbf{x}_i - \mathbf{y}_j\|_2^2 \cdot \tilde{W}_{ij} + \varepsilon \sum_i^m \sum_j^n \frac{1}{(\hat{f}_i)^{\,\theta}\,(\hat{g}_j)^{\,\theta}} \cdot \tilde{W}_{ij} \log\tilde{W}_{ij} \\
            \text{s.t.} \quad & \frac{1}{n}\sum_{j=1}^n \frac{\tilde{W}_{ij}}{(\hat{f}_i)^{\,\theta}\,(\hat{g}_j)^{\,\theta} } = \frac{1}{\sqrt{S^{(\theta)}}(\hat{f}_i)^{\,\theta}},
            \quad
            \frac{1}{m}\sum_{i=1}^m \frac{\tilde{W}_{ij}}{(\hat{f}_i)^{\,\theta}\,(\hat{g}_j)^{\,\theta}}
            = \frac{\sqrt{S^{(\theta)}}}{(\hat{g}_j)^{\,\theta}}, \quad \forall i \in [m], \forall j \in [n], 
            \end{aligned}
        \end{equation}
        where $S^{(\theta)}$ is defined in~\eqref{eq:S_scale}. 
    \end{proposition}

    Compared with the standard EOT formulation in~\eqref{eq:EOT_problem}, the problem in~\eqref{eq:density_invariant_opt} reweights the contribution of each pair $(\mathbf{x}_i, \mathbf{y}_j)$ in both the transport cost and the entropic regularizer by $1/(\hat{f}_i\hat{g}_j)^{\,\theta}$. The marginal constraints are reweighted analogously: the mass at each data point is scaled inversely by its local density: $1/(\hat{f}_i)^{\,\theta}$ for each $\mathbf{x}_i$ and $1/(\hat{g}_j)^{\,\theta}$ for each $\mathbf{y}_j$. This reweighting scheme assigns more mass to points in sparse regions and less mass to points in dense regions, compensating for non-uniform sampling. The scaling factor $\sqrt{S}$ on the right-hand side ensures that total mass is conserved across the two reweighted marginals.

    The reweighting in Proposition~\ref{prop:density_invariant_eot} is motivated by the goal of discounting the influence of sampling density from the transport plan, so that the resulting transport plan approximates the plan one would obtain if both datasets were drawn according to density-reweighted reference measures from their respective manifolds. In what follows, we first introduce the family of population-level EOT plans $\mathcal{W}^{(\theta)}_{\varepsilon}$ under density-reweighted reference measures. We then show that under mild regularity conditions, $\mathbf{W}^{(\theta)}$ in~\eqref{eq:EOT_solution_form} computed with fixed $\theta \in [0,1]$ and known sampling densities converges uniformly to $\mathcal{W}^{(\theta)}_{\varepsilon}$, up to a global scaling constant, with high probability as $m, n \to \infty$. We subsequently address the practically relevant setting in which the sampling densities are unknown and must be estimated from data.

    We next formalize our assumptions on the data geometry, sampling densities, and dataset sizes.
    \begin{assump}
    \label{assump:manifold_and_density}
        $\mathcal{M}_x$ and $\mathcal{M}_y$ are smooth, compact, boundaryless Riemannian manifolds isometrically embedded in $\mathbb{R}^{D}$, with intrinsic dimension $d_1$ and $d_2$, respectively. The sampling densities $\mu$ and $\nu$ are strictly positive and belong to $\mathcal{C}^3(\mathcal{M}_x)$ and $\mathcal{C}^3(\mathcal{M}_y)$, respectively.
    \end{assump}
    Assumption~\ref{assump:manifold_and_density} is a standard regularity condition in manifold learning, with compactness of the manifolds and smoothness of the densities ensuring that finite-sample estimators concentrate around their population-level counterparts at a controlled rate.
 
    \begin{assump}
    \label{assump:m_and_n}
        $n \geq m \geq n^\gamma$ for some constant $\gamma  \in (0,1)$.
    \end{assump} 
    We note that Assumption~\ref{assump:m_and_n} generalizes to $n \geq m \geq c n^{\gamma}$ for any constant $c > 0$; we set $c = 1$ for simplicity.  
    To state our results, we define the $\theta$-reweighted density functions as
    \begin{equation}\label{eq:tilt_density}
        \mu^{(\theta)}(\mathbf{x}) := \frac{\left(\mu(\mathbf{x})\right)^{1-\theta}}{\int_{\mathcal{M}_x} \left(\mu(\mathbf{x})\right)^{1-\theta}\, d\omega(\mathbf{x})}, \quad 
        \nu^{(\theta)}(\mathbf{y}) := \frac{\left(\nu(\mathbf{y})\right)^{1-\theta}}{\int_{\mathcal{M}_y} \left(\nu(\mathbf{y})\right)^{1-\theta}\, d\eta(\mathbf{y})},
    \end{equation}
    where, as $\theta$ increases from $0$ to $1$, the reweighted densities interpolate between 
    the sampling densities $\mu$, $\nu$ and the uniform densities on their respective manifolds. We define the family of corresponding population-level EOT plans $\mathcal{W}^{(\theta)}_\varepsilon: \mathcal{M}_x \times \mathcal{M}_y \to \mathbb{R}_+$ and the Gaussian kernel $\mathcal{K}_\varepsilon: \mathcal{M}_x \times \mathcal{M}_y \to \mathbb{R}_+$ as
    \begin{equation}
    \label{eq:W_kernel}
    \mathcal{W}^{(\theta)}_\varepsilon(\mathbf{x},\mathbf{y}) = u_\varepsilon^{(\theta)}(\mathbf{x}) \mathcal{K}_\varepsilon(\mathbf{x},\mathbf{y}) v_\varepsilon^{(\theta)}(\mathbf{y}), \qquad \mathcal{K}_\varepsilon(\mathbf{x},\mathbf{y}) = \exp\left\{ \frac{-\|\mathbf{x}-\mathbf{y}\|_2^2}{\varepsilon} \right\},
    \end{equation}
    where $u_\varepsilon^{(\theta)}$ and $v_\varepsilon^{(\theta)}$ are functions that solve the integral equations
    \begin{equation}
    \label{eq:integral_eq}
    \begin{cases}
    \displaystyle \int_{\mathcal{M}_y} \mathcal{W}^{(\theta)}_\varepsilon(\mathbf{x},\mathbf{y}) \, \nu^{(\theta)}(\mathbf{y}) \, d\eta(\mathbf{y}) 
    = \int_{\mathcal{M}_y} u_\varepsilon^{(\theta)}(\mathbf{x}) \mathcal{K}_\varepsilon(\mathbf{x},\mathbf{y}) v_\varepsilon^{(\theta)}(\mathbf{y}) \, \nu^{(\theta)}(\mathbf{y})\, d\eta(\mathbf{y}) = 1, \quad \forall \mathbf{x} \in \mathcal{M}_x; \\[8pt]
    \displaystyle \int_{\mathcal{M}_x} \mathcal{W}^{(\theta)}_\varepsilon(\mathbf{x},\mathbf{y}) \, \mu^{(\theta)}(\mathbf{x})\, d\omega(\mathbf{x}) 
    = \int_{\mathcal{M}_x} u_\varepsilon^{(\theta)}(\mathbf{x}) \mathcal{K}_\varepsilon(\mathbf{x},\mathbf{y}) v_\varepsilon^{(\theta)}(\mathbf{y})\, \mu^{(\theta)}(\mathbf{x}) \, d\omega(\mathbf{x}) = 1, \quad \forall \mathbf{y} \in \mathcal{M}_y. 
    \end{cases}
    \end{equation}
    We note $\mathcal{W}^{(\theta)}_\varepsilon$ is doubly stochastic with respect to the density-reweighted measures $\mu^{(\theta)} \, d\omega$ and $\nu^{(\theta)} \, d\eta$. The scaling functions $u_\varepsilon^{(\theta)}$ and $v_\varepsilon^{(\theta)}$ are guaranteed to exist as positive, continuous functions on $\mathcal{M}_x$ and $\mathcal{M}_y$, respectively, and are unique up to a positive multiplicative constant~\cite{knopp_theory}.

    We now relate the transport plan $\mathbf{W}^{(\theta)}$ in~\eqref{eq:EOT_solution_form}, computed with $\theta \in [0,1]$ and known sampling densities, to the population-level EOT map $\mathcal{W}^{(\theta)}_{\varepsilon}$ defined in~\eqref{eq:W_kernel}, in Theorem~\ref{thm:finite_population_density_known} below. The proof is in Supplement~\ref{sec:proof density known}.

    \begin{theorem}\label{thm:finite_population_density_known}
        Suppose Algorithm~\ref{alg:EOT_invariant} is applied with exact sampling densities $\hat{f}_i = \mu(\mathbf{x}_i)$ and $\hat{g}_j = \nu(\mathbf{y}_j)$, $\forall i \in [m]$, $\forall j \in [n]$, and let $\mathbf{W}^{(\theta)} \in \mathbb{R}^{m\times n}_+$ denote the resulting transport plan in~\eqref{eq:EOT_solution_form}. Then, under Assumptions~\ref{assump:manifold_and_density} and~\ref{assump:m_and_n}, there exist positive constants $\tau_0, n_0(\varepsilon), C'(\varepsilon)$ (which may additionally depend on $\mathcal{M}_x, \mathcal{M}_y, \mu, \nu, \theta$, and $\gamma$) such that for all $\tau \geq \tau_0$, $n \geq n_0(\varepsilon, \tau)$, and $m \geq n^{\gamma}$, with probability at least $1 - n^{-\tau}$,
        \begin{equation}\label{eq:W_ij2W_kernel}
            \left| C^{(\theta)}\, W_{ij}^{(\theta)} - \mathcal{W}^{(\theta)}_{\varepsilon}(\mathbf{x}_i,\mathbf{y}_j) \right| 
            \leq \tau C'(\varepsilon) \sqrt{\frac{\log m}{m}}
        \end{equation}
        holds simultaneously for all $i \in [m]$ and $j \in [n]$, where $C^{(\theta)} = \sqrt{\int_{\mathcal{M}_x} \mu(\mathbf{x})^{1-\theta}\, d\omega(\mathbf{x}) \cdot \int_{\mathcal{M}_y} \nu(\mathbf{y})^{1-\theta}\, d\eta(\mathbf{y})}$. 
    \end{theorem}

    \begin{wrapfigure}{r}{0.42\textwidth}
        \centering
        \includegraphics[width=0.4\textwidth]{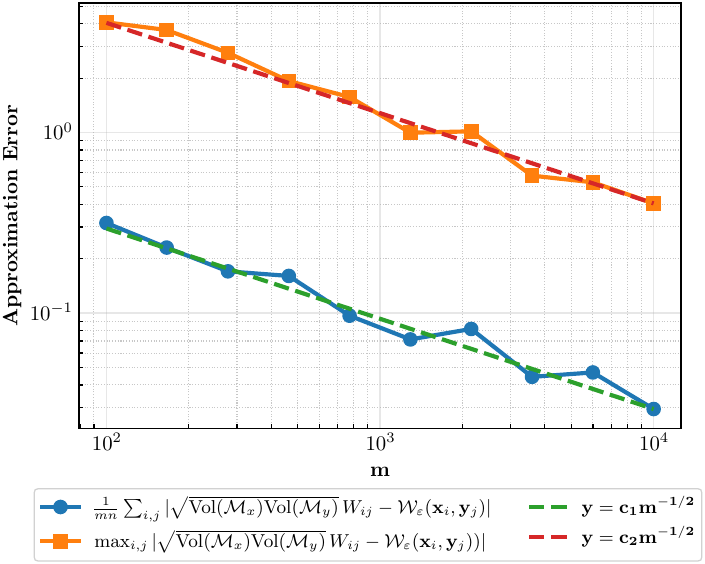}
        \caption{Empirical evaluation of the convergence rate of the reweighted EOT plan. Both error metrics align with~\eqref{eq:W_ij2W_kernel}. Here $C^{(\theta=1)} = \sqrt{\mathrm{Vol}(\mathcal{M}_x)\,\mathrm{Vol}(\mathcal{M}_y)}$, where $\mathrm{Vol}(\mathcal{M}_x)$ and $\mathrm{Vol}(\mathcal{M}_y)$ denote the Riemannian volumes of $\mathcal{M}_x$ and $\mathcal{M}_y$, respectively. See Supplement~\ref{sec:demo} for the experiment details. }
        \vspace{-8pt}
        \label{fig:demo_convergence}
    \end{wrapfigure}

    Theorem~\ref{thm:finite_population_density_known} shows that for any fixed $\varepsilon$ and sufficiently large $m$ and $n$, the entries $W_{ij}^{(\theta)}$ concentrate around $\mathcal{W}^{(\theta)}_{\varepsilon}(\mathbf{x}_i,\mathbf{y}_j)/C^{(\theta)}$ simultaneously for all pairs $(i,j)$ with high probability. It further provides an explicit probabilistic bound on the approximation error that depends on the smaller dataset size $m$ and decays at the standard sample-to-population rate $m^{-1/2}$ up to logarithmic factors. In particular, Theorem~\ref{thm:finite_population_density_known} demonstrates that setting $\theta = 1$ in Algorithm~\ref{alg:EOT_invariant} effectively removes the dependence of $\mathbf{W}^{(\theta)}$ on sampling density, recovering a transport plan that depends solely on the underlying geometry of the manifolds. We numerically demonstrate the convergence rate using the same simulation setup as Figure~\ref{fig:demo} with $\theta = 1$. As shown in Figure~\ref{fig:demo_convergence}, the approximation error in both normalized $\ell_1$ and $\ell_\infty$ norms
    decay approximately like $m^{-1/2}$. Although this stylized example has boundary and piecewise-smooth sampling density (hence outside the formal assumptions of Theorem~\ref{thm:finite_population_density_known}), it illustrates the predicted convergence behavior.

    In practice, the sampling densities $\mu$ and $\nu$ are usually unknown and need to be estimated from data. A natural approach is kernel density estimation (KDE): given $N$ points $\{\mathbf{z}_i\}_{i=1}^N$ in $\mathbb{R}^D$, the KDE with bandwidth $h > 0$ is given by 
    \begin{equation}\label{eq:KDE_classical}
        \hat{\rho}(\mathbf{z}) = \frac{1}{N h^{d/2}}\sum_{i=1}^{N} K\!\left(\frac{\mathbf{z} - \mathbf{z}_i}{\sqrt{h}}\right),
    \end{equation}
    where $K$ is a kernel function. If the points are sampled from a $d$-dimensional manifold embedded in $\mathbb{R}^D$, and $h\rightarrow 0$ sufficiently slowly as $N\rightarrow\infty$, then the KDE converges pointwise to the underlying sampling density. Moreover, the bandwidth can be tuned optimally to achieve a mean squared error of $\mathcal{O}(N^{-4/(d+4)})$~\cite{kde_ref}. Since the form of the KDE~\eqref{eq:KDE_classical} depends explicitly on the intrinsic dimension $d$, which is typically unknown in practice, we consider a mean-normalized KDE whose construction does not require knowledge of $d$. 
    \begin{definition}\label{def:self_normalized_kde}
        For a fixed bandwidth $h > 0$ and kernel $K_h(x,y) = \exp(-\|x-y\|_2^2/h)$, the mean-normalized density estimator at any $\mathbf{x}_i \in \mathbf{X} = \{\mathbf{x}_i\}_{i=1}^m \subset \mathcal{M}$ is defined as
        \begin{equation}\label{eq:self_KDE}
            \hat{f}(\mathbf{x}_i) := \frac{\hat{q}(\mathbf{x}_i)}{\hat{Z}}, \qquad \qquad 
            \hat{q}(\mathbf{x}_i) := \frac{1}{m-1}\sum_{j \neq i}^m K_h(\mathbf{x}_i, \mathbf{x}_j), \qquad \qquad  
            \hat{Z} := \frac{1}{m}\sum_{i=1}^m \hat{q}(\mathbf{x}_i).
        \end{equation}
    \end{definition}
    This construction is motivated by the observation that the optimal $\mathbf{W}^{(\theta)}$ in Proposition~\ref{prop:density_invariant_eot} is invariant to per-dataset multiplicative scalings of the density estimates: if $\hat{f}_i$ and $\hat{g}_j$ approximate $\mu(\mathbf{x}_i)$ and $\nu(\mathbf{y}_j)$ each up to a global constant, the resulting transport plan is rescaled by a positive global constant, leaving the correspondence structure intact. 

    We now establish that the transport plan $\mathbf{W}^{(\theta)}$ computed with a fixed $\theta \in [0,1]$ using the mean-normalized KDE converges, up to a global scaling constant, to the population EOT plan $\mathcal{W}^{(\theta)}_{\varepsilon}$ from~\eqref{eq:W_kernel}. The proof is given in Supplement~\ref{sec:proof density estimated}.

    \begin{theorem}\label{thm:finite_population_density_estimated}
        Let $\hat{\mathbf{f}}$ and $\hat{\mathbf{g}}$ be the mean-normalized density estimators from~\eqref{eq:self_KDE} applied to $\mathbf{X}$ and $\mathbf{Y}$ with bandwidths $h_1$ and $h_2$, respectively. Suppose Algorithm~\ref{alg:EOT_invariant} uses $\hat{f}_i = \hat{\mathbf{f}}(\mathbf{x}_i)$ and $\hat{g}_j = \hat{\mathbf{g}}(\mathbf{y}_j)$. Then, under Assumptions~\ref{assump:manifold_and_density} and~\ref{assump:m_and_n}, there exist positive constants $\tau_0$ and $C'(\varepsilon)$ such that the following holds. For every $\tau \ge \tau_0$, one can choose positive constants $n_0(\varepsilon,\tau)$, $\kappa_1(\varepsilon,\tau)$, $\kappa_2(\varepsilon,\tau)$, $h_1'(\varepsilon,\tau) \le 1$, and $h_2'(\varepsilon,\tau) \le 1$ (which may additionally depend on $\mathcal{M}_x, \mathcal{M}_y, \mu, \nu, \theta$, and $\gamma$) such that for all $n \ge n_0(\varepsilon,\tau)$, $m \ge n^{\gamma}$ (under Assumption~\ref{assump:m_and_n}), and all bandwidths $h_1, h_2$ satisfying
        \begin{equation}\label{eq:bandwidth_window}
            \left( \frac{\kappa_1 \log m}{m} \right)^{1/d_1} \le h_1 \le h_1',
            \qquad
            \left( \frac{\kappa_2 \log n}{n} \right)^{1/d_2} \le h_2 \le h_2',
        \end{equation}
        we have with probability at least $1 - n^{-\tau} - 4n^{-2\gamma}$,
        \begin{equation}\label{eq:W_ij2W_kernel_est}
            \left| C^{(\theta)'} W_{ij}^{(\theta)} - \mathcal{W}^{(\theta)}_{\varepsilon}(\mathbf{x}_i,\mathbf{y}_j) \right|
            \le \tau C'(\varepsilon) \left( h_1 + \sqrt{\frac{\log m}{m h_1^{d_1}}} + h_2 + \sqrt{\frac{\log n}{n h_2^{d_2}}} \right)
        \end{equation}
        holds simultaneously for all $i \in [m]$ and $j \in [n]$, where \sloppy$C^{'(\theta)} = (\|\mu\|_{L^2}\|\nu\|_{L^2})^{\theta} \cdot  \sqrt{\int_{\mathcal{M}_x} \mu(\mathbf{x})^{1-\theta}\, d\omega(\mathbf{x}) \cdot \int_{\mathcal{M}_y} \nu(\mathbf{y})^{1-\theta}\, d\eta(\mathbf{y})};$ with $\|\mu\|_{L^2} = \sqrt{\int_{\mathcal{M}_x} \mu^2\,d\omega}$ and $\|\nu\|_{L^2} = \sqrt{\int_{\mathcal{M}_y} \nu^2\,d\eta}$.
    \end{theorem}

    Theorem~\ref{thm:finite_population_density_estimated} extends Theorem~\ref{thm:finite_population_density_known} to the practical setting where the sampling densities $\mu$ and $\nu$ are unknown and must be estimated from the data. The probabilistic bound in~\eqref{eq:W_ij2W_kernel_est} depends explicitly on the bandwidths $h_1$ and $h_2$, the sample sizes $m$ and $n$, and the intrinsic dimensions $d_1$ and $d_2$. Choosing $h_1 \asymp \left(\frac{\log m}{m}\right)^{1/(d_1+2)}$ and $h_2 \asymp \left(\frac{\log n}{n}\right)^{1/(d_2+2)}$ balances the bias-like and variance-like terms for each density estimator in~\eqref{eq:W_ij2W_kernel_est}, yielding a convergence rate of \sloppy $\mathcal{O}\!\left( \left(\frac{\log m}{m}\right)^{1/(d_1+2)} + \left(\frac{\log n}{n}\right)^{1/(d_2+2)} \right)$. A data-driven approach to set the bandwidths is described in the next section.

\section{Experiments}\label{sec:exp_main}

    In this section, we consider the task of aligning two datasets generated by sampling from the same underlying manifold up to a constant shift, but with substantially different sampling densities. We demonstrate that our density-reweighted EOT framework recovers geometrically faithful correspondences under this setting, while standard EOT and UOT (across multiple hyperparameter configurations) fail systematically.

    The experiment involves two manifolds $\mathcal{M}_x$ and $\mathcal{M}_y$, each composed of the union of two circular arcs embedded in $\mathbb{R}^2$. Specifically, $\mathcal{M}_x$ consists of a full unit circle centered at the origin and a semi-circle of radius $2$ centered at $(4, 1)$. The manifold $\mathcal{M}_y$ is obtained by translating $\mathcal{M}_x$ by $(3, 5)^\top$. Datasets $\mathbf{X} = \{\mathbf{x}_i\}_{i=1}^m$ and $\mathbf{Y} = \{\mathbf{y}_j\}_{j=1}^n$ are sampled from $\mathcal{M}_x$ and $\mathcal{M}_y$ respectively, with opposing sampling densities over their components. Specifically, each $\mathbf{x}_i$ is drawn from the unit circle with probability $0.8$ and from the semi-circle with probability $0.2$, while each $\mathbf{y}_j$ reverses these weights: drawn from the unit circle with probability $0.2$ and from the semi-circle with probability $0.8$. Additional density variations are imposed within each component; see Supplement~\ref{sec:sim_detail} for details and Figure~\ref{fig:sim_result}(a) for a visualization of the sampling densities on each manifold. 

    We compare our approach against standard EOT and UOT, with the latter tested across multiple hyperparameter configurations. Performance is measured by the geometric fidelity of the correspondences encoded in the resulting transport plans, using the $\overline{R_k}$ metric:
    \begin{equation}\label{eq:recall_k}
        \overline{R_k}
        =
        \frac{1}{2}\left(
        R_k(\mathbf{X},\mathbf{Y})
        +
        R_k(\mathbf{Y},\mathbf{X})
        \right),
        \end{equation}
        where
        \begin{equation}
        R_k(\mathbf{S},\mathbf{T})
        =
        \frac{1}{C_{\mathbf{S}}}
        \sum_{c=1}^{C_{\mathbf{S}}}
        \frac{1}{n_c}
        \sum_{i=1}^{n_c}
        \frac{
        \left|
        \mathcal{K}^{(\mathbf{W})}_k(\mathbf{s}_i,\mathbf{T})
        \cap
        \mathcal{K}^{(\mathrm{true})}_k(\mathbf{s}_i,\mathbf{T})
        \right|
        }{k}.
    \end{equation}
    Here $(\mathbf{S},\mathbf{T})$ is either $(\mathbf{X},\mathbf{Y})$ or $(\mathbf{Y},\mathbf{X})$, with $\mathbf{S}$ and $\mathbf{T}$ denoting the source and target datasets, respectively, and $\mathbf{s}_i$ denoting the $i$-th point in the $c$-th component of $\mathbf{S}$. Here $C_{\mathbf{S}}$ is the number of disjoint components of the manifold from which $\mathbf{S}$ is sampled (e.g., $C_{\mathbf{X}}=C_{\mathbf{Y}}=2$), $n_c$ is the number of points in $\mathbf{S}$ belonging to the $c$-th component, $\mathcal{K}^{(\mathrm{true})}_k(\mathbf{s}_i,\mathbf{T})$ is the set of true $k$ nearest neighbors of $\mathbf{s}_i$ in $\mathbf{T}$, and $\mathcal{K}^{(\mathbf{W})}_k(\mathbf{s}_i,\mathbf{T})$ is the set of $k$ points in $\mathbf{T}$ with the highest coupling weights from $\mathbf{s}_i$ under plan $\mathbf{W}^{(\theta)}$.
    The metric macro-averages over components so that each geometric component contributes equally regardless of sampling density imbalances, and averages symmetrically over both transfer directions ($\mathbf{X}\rightarrow\mathbf{Y}$ and $\mathbf{Y}\rightarrow\mathbf{X}$) so that the score does not favor either direction. In our experiment, $\mathcal{K}^{(\mathrm{true})}_k$ is computed by translating $\mathbf{Y}$ back by the known shift to obtain $\mathbf{Y}'$, and finding nearest neighbors between $\mathbf{X}$ and $\mathbf{Y}'$ via pairwise Euclidean distances.

    We next describe the hyperparameter configurations for each method, beginning with parameters specific to each approach, after which we describe how the entropic regularization parameter $\varepsilon$ is selected. Our method requires two kernel bandwidths for the density estimation in the mean-normalized KDE~\eqref{eq:self_KDE}, which we select via a data-driven procedure combining bootstrap resampling~\cite{bootstrap} with Lepski's method~\cite{lepski}. Lepski's method considers a geometric grid of candidate bandwidths and selects the largest bandwidth $h$ for which the kernel density estimate does not differ significantly from those at any smaller bandwidth in the grid (i.e., the discrepancy remains within the expected stochastic fluctuation). Intuitively, it chooses the coarsest resolution at which no additional statistically significant structure is revealed at finer scales: for any $\eta < h$ in the grid, the difference of density estimates $\hat{f}_h - \hat{f}_\eta$ is small enough to be explained by stochastic variability rather than systematic bias. Bootstrap resampling is used to empirically quantify this stochastic variability. The full procedure is detailed in Algorithm~\ref{alg:bootstrap_lepski} in Supplement~\ref{sec:sim_detail}. For UOT, 
    we set $\lambda_1 = \lambda_2 =: \lambda$ in~\eqref{eq:UOT_problem} and test a grid of values for $\lambda$. All three methods require choosing an entropic regularization parameter $\varepsilon$, which controls the diffusiveness of the resulting transport plan. We select $\varepsilon$ for each method (with other hyperparameters fixed) by maximizing the number of mutual $k$-nearest neighbor (MKNN) pairs in the resulting plan, which serves as a proxy for alignment consistency. Specifically, $\mathbf{x}_i$ and $\mathbf{y}_j$ form an MKNN pair if each is among the top-$k$ maximum-coupling points of the other, reflecting a mutually agreed-upon correspondence. When $\varepsilon$ is too small, the transport plan is overly sparse and few MKNN pairs exist; when $\varepsilon$ is too large, the plan becomes nearly uniform and MKNN pairs are again scarce as correspondences lose specificity. Selecting the $\varepsilon$ that maximizes the number of MKNN pairs at a prespecified $k$ thus strikes a balance between these two extremes, yielding the greatest alignment consistency. See Supplement~\ref{sec:sim_detail} for details. 

    \begin{figure}
        \centering
        \includegraphics[width=\linewidth]{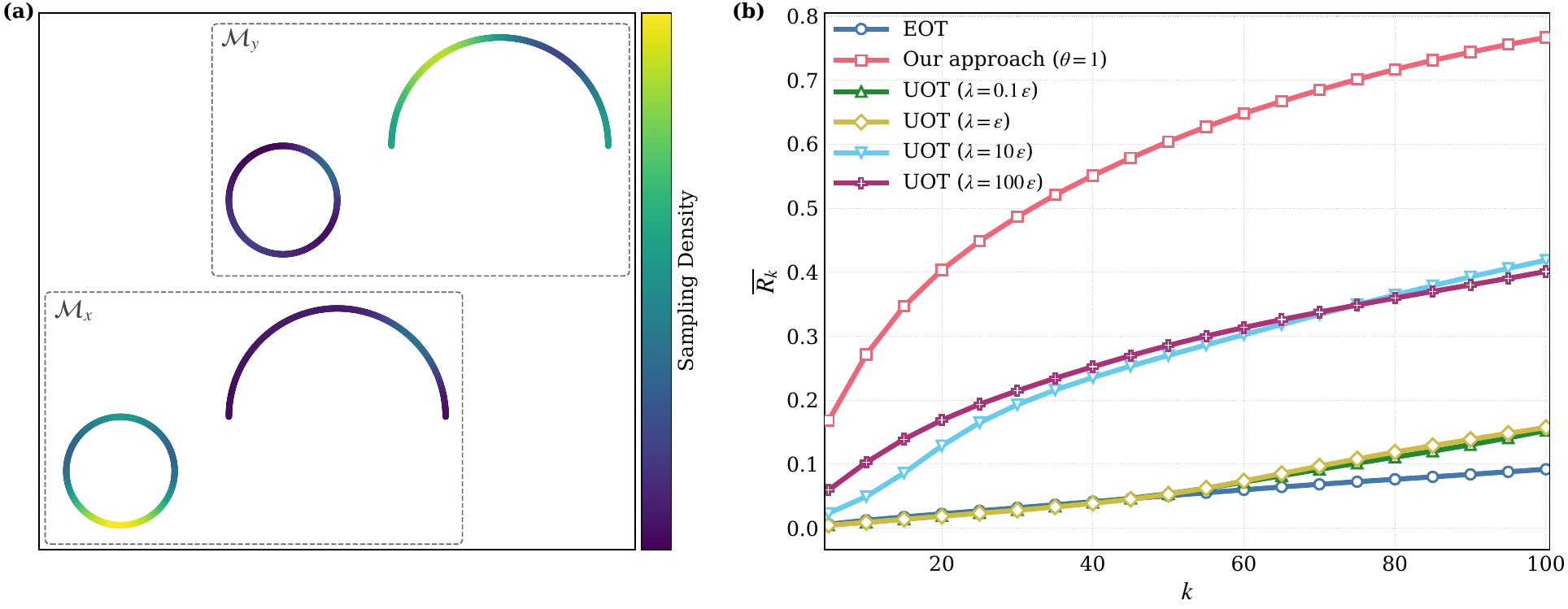}
        \caption{\textbf{Synthetic alignment under opposing sampling densities.} \textbf{(a)} Visualization of $\mathcal{M}_x$ and $\mathcal{M}_y$ colored by sampling density (yellow: high; purple: low). \textbf{(b)} Alignment quality $\overline{R_k}$ in~\eqref{eq:recall_k} as a function of $k \in [5, 100]$ for all methods and hyperparameter configurations examined.}
        \label{fig:sim_result}
    \end{figure}

    In Figure~\ref{fig:sim_result}(b), we report $\overline{R_k}$  averaged over 10 independent trials for all methods across $k \in [5, 100]$. Our approach ($\theta = 1$) substantially outperforms all other methods at every neighborhood size, with $\overline{R_k}$ consistently exceeding that of the best UOT configuration by an approximate factor of $2\times$ or more, and that of standard EOT by nearly an order of magnitude across all examined $k$. Standard EOT performs worst across all $k$, confirming that correspondences derived without accounting for heterogeneous sampling densities are systematically confounded by local data density. UOT partially mitigates this failure by relaxing the marginal constraints, but its performance varies substantially across hyperparameter configurations. Such sensitivity to $\lambda$ may limit the practical utility of UOT, as dataset alignment is typically unsupervised and provides no natural criterion for choosing the degree of marginal relaxation.

    \section*{Acknowledgments}
    The authors used Anthropic Claude Sonnet and OpenAI ChatGPT to assist with language editing, exposition, and technical consistency checks. This work was supported by NIH grants UM1PA051410, U54AG076043, U54AG079759, U01DA053628, P50CA121974, and R33DA047037.

\begin{appendices}
    \section{Details of Experiments and Additional Results}\label{sec:exp_detail}

    \subsection{Experimental Setup for Figures~\ref{fig:demo}, \ref{fig:demo_alpha}, and \ref{fig:demo_convergence}}~\label{sec:demo}
        We consider a two-dimensional toy example to illustrate the effect of sampling densities on EOT plans. The source dataset $\mathbf{X} = \{\mathbf{x}_i\}_{i=1}^n = \{(x_{i1}, x_{i2})\}_{i=1}^n$ lies on the line $x_{i2} = x_{i1}$, and the target dataset $\mathbf{Y} = \{\mathbf{y}_j\}_{j=1}^m = \{(y_{j1}, y_{j2})\}_{j=1}^m$ lies on the curve $y_{j2} = 2 + y_{j1} + 0.5\,y_{j1}^2$. Both datasets share the first-coordinate domain $[0,1]$ but differ substantially in sampling density. Specifically, for the dataset $\mathbf{X}$, the first coordinates $x_{i1}$ are drawn uniformly from $[0, 0.5]$ with probability $0.9$ and uniformly from $[0.5, 1]$ with probability $0.1$; the dataset $\mathbf{Y}$ is generated analogously with the probabilities reversed. The second coordinate of data point in each dataset is then determined by the respective curve equation applied to its first coordinate.

        For computing the transport plan, standard EOT imposes uniform empirical measures as marginal constraints (i.e, the row and columns sums are $n \cdot \mathbf{1}_m$ and $m\cdot \mathbf{1}_n$, respectively), whereas our method solves for $\mathbf{W}^{(\theta)}$ via Algorithm~\ref{alg:EOT_invariant} using the ground-truth sampling densities. The entropic regularization parameter is fixed at $\varepsilon = 5 \times 10^{-2}$ for both methods.

        \textbf{Figure 1.} The sample sizes are $n = 3{,}000$ and $m = 2{,}000$. We set $\theta = 1$ in Algorithm~\ref{alg:EOT_invariant}.

        \textbf{Figure 2.} The sample sizes are $n = 3{,}000$ and $m = 2{,}000$. We vary $\theta$ over a uniform grid on $[0, 1]$ and compute the corresponding transport plan $\mathbf{W}^{(\theta)}$ for each value via Algorithm~\ref{alg:EOT_invariant}.

        \textbf{Figure 3.} The population-level EOT plan $\mathcal{W}_{\varepsilon}$ in~\eqref{eq:W_kernel} is computed via the following interpolation procedure. We first construct a fine grid of $N = 10{,}000$ points placed equally spaced with respect to arc length on each manifold, and solve standard EOT on this $N \times N$ grid via the Sinkhorn--Knopp algorithm, yielding scaling vectors $\hat{\mathbf{u}}, \hat{\mathbf{v}} \in \mathbb{R}^N$. These vectors are then extended to continuous scaling functions $\tilde{u}_\varepsilon$ and $\tilde{v}_\varepsilon$ via cubic spline interpolation in the first coordinate, yielding approximations to ${u}_\varepsilon$ and ${v}_\varepsilon$ defined in~\eqref{eq:integral_eq}. For each dataset, the approximate population-level EOT plan is then calculated entrywise as
        \begin{equation}
            \tilde{\mathcal{W}}_\varepsilon\!\bigl(\mathbf{x}_i, \mathbf{y}_j\bigr) = \tilde{u}_\varepsilon(x_{i1})\exp\!\Bigl(-\tfrac{\|\mathbf{x}_i - \mathbf{y}_j\|^2}{\varepsilon}\Bigr)\tilde{v}_\varepsilon(y_{j1}), \qquad i \in [m],\ j \in [n].
        \end{equation}
    
        To generate the convergence curve, we vary $m$ over a log-uniform grid of $10$ values on $[100, 10{,}000]$ and set $n = 2m$. For each value of $m$, the experiment is repeated $20$ times independently and the results are averaged to reduce variability. The two error metrics considered are:
        \begin{equation}
        \begin{aligned}
            E_1(m) &= \frac{1}{mn}\sum_{i=1}^{m}\sum_{j=1}^{n}
            \Bigl|\sqrt{\mathrm{Vol}(\mathcal{M}_x)\,\mathrm{Vol}(\mathcal{M}_y)}\,W_{ij}
            - \mathcal{W}_\varepsilon(\mathbf{x}_i,\mathbf{y}_j)\Bigr|, \\
            E_\infty(m) &= \max_{i,j}
            \Bigl|\sqrt{\mathrm{Vol}(\mathcal{M}_x)\,\mathrm{Vol}(\mathcal{M}_y)}\,W_{ij}
            - \mathcal{W}_\varepsilon(\mathbf{x}_i,\mathbf{y}_j)\Bigr|,
        \end{aligned}
        \end{equation}
        where $\mathrm{Vol}(\mathcal{M}_x)$ and $\mathrm{Vol}(\mathcal{M}_y)$ are computed analytically via arc-length integration.

    \subsection{Estimation Performance of mean-normalized KDE}\label{sec:self_kde_prop}

    \begin{lemma}
    \label{lem:self-normalizing kde}
    Let $\mathcal{M}_x \subset \mathbb{R}^{D}$ be a smooth, compact, 
    boundaryless $d$-dimensional Riemannian manifold isometrically embedded in
    $\mathbb{R}^{D}$, let $\mu \in \mathcal{C}^{3}(\mathcal{M}_x)$ be a positive
    probability density on $\mathcal{M}_x$, and let
    $\mathbf{X} = \{\mathbf{x}_i\}_{i=1}^{m} \overset{\mathrm{iid}}{\sim} \mu$. Let
    $\hat f(\mathbf{x}_i)$ be the mean-normalized density estimator defined
    in~\eqref{eq:self_KDE} with bandwidth $h$. Then there exist positive constants
    $h_0(\mathcal{M}_x, \mu) \le 1$, $\kappa(\mathcal{M}_x, \mu)$,
    $C_0(\mathcal{M}_x, \mu)$, and $m_0(h_0, \kappa) \ge 2$ such that for all
    $m \ge m_0(h_0,\kappa)$ and all fixed $h$ with
    \begin{equation}
    \label{eq:A1-window}
        \left( \frac{\kappa \log m}{m} \right)^{1/d}
        \;\le\; h \;\le\; h_0(\mathcal{M}_x, \mu),
    \end{equation}
    with probability at least $1 - 2 m^{-2}$,
    \begin{equation}
    \label{eq:A1-bound}
        \left| \hat f(\mathbf{x}_i) - \frac{\mu(\mathbf{x}_i)}{(\|\mu\|_{L^2})^{2}} \right|
        \;\le\; C_0 \left( h + \sqrt{\frac{\log m}{m h^{d}}} \right)
        \qquad \text{for all } i \in [m],
    \end{equation}
    where
    $\|\mu\|_{L^2} = \bigl( \int_{\mathcal{M}_x} \mu(\mathbf{x})^{2}\, d\omega(\mathbf{x}) \bigr)^{1/2}$.
    \end{lemma}
    
    The proof is in Appendix~\ref{sec:proof_density_estimated}.
    
    \subsection{Experimental Setup for Figure~\ref{fig:sim_result}}\label{sec:sim_detail}

    \paragraph{Sampling Density.}The sampling density on each component is parameterized by $\phi$ as follows:
    \begin{itemize}
        \item \textbf{Unit circle of $\mathcal{M}_x$:}
            $\phi = \tilde{\phi} \bmod 2\pi, \text{ where } \tilde{\phi} \sim 0.3\cdot\mathcal{N}(0.5\pi,\;0.8^2)
                        + 0.7\cdot\mathcal{N}(1.5\pi,\;1.0^2)$;
        \item \textbf{Semi-circle of $\mathcal{M}_x$:}
            $\phi = \tilde{\phi} \bmod \pi, \text{ where } \tilde{\phi} \sim 0.6\cdot\mathcal{N}(0.25\pi,\;0.3^2)
                        + 0.4\cdot\mathcal{N}(0.65\pi,\;0.8^2)$;
        \item \textbf{Unit circle of $\mathcal{M}_y$:}
            $\phi = \tilde{\phi} \bmod 2\pi, \text{ where } \tilde{\phi} \sim 0.5\cdot\mathcal{N}(0.2\pi,\;0.5^2)
                        + 0.5\cdot\mathcal{N}(1.3\pi,\;0.9^2)$;
        \item \textbf{Semi-circle of $\mathcal{M}_y$:}
            $\phi = \tilde{\phi} \bmod \pi, \text{ where } \tilde{\phi} \sim 0.25\cdot\mathcal{N}(0.1\pi,\;0.4^2)
                        + 0.75\cdot\mathcal{N}(0.7\pi,\;0.6^2)$,
    \end{itemize}
    where $\mathcal{N}(\mu, \sigma^2)$ denotes a Gaussian distribution with mean
    $\mu$ and variance $\sigma^2$. Given $\phi$, a point on a circle or semi-circle of radius $r$ centered at
    $\mathbf{c}$ is given by $\mathbf{c} + r(\cos\phi,\, \sin\phi)^\top$. We note that the densities of $\phi$ on each component are chosen arbitrarily and are not tuned to favor any particular method.

    \paragraph{KDE bandwidth Selection.} The KDE bandwidth for each dataset $\mathbf{X}$ and $\mathbf{Y}$ is selected independently via Algorithm~\ref{alg:bootstrap_lepski} with the following settings: $k = 20$, $2h_{\min}^2 = 10^{-2}$, $2h_{\max}^2 = 1$, $m = 50$, $\alpha = 0.05$, and $C = 1$.
 
    \paragraph{Entropic Regularization Parameter $\varepsilon$.} For each method, $\varepsilon$ is chosen to maximize the number of MKNN pairs ($K = 50$) in the resulting transport plan, via golden-section search over $[0, 5]$, a range chosen based on the scale of the data. The search terminates when the active interval shrinks below $0.10$ or after at most $20$ iterations, whichever comes first.

    \paragraph{Figure~\ref{fig:sim_result}.}We fix the sample sizes at $m = n = 2{,}000$ and compute transport plans using standard EOT, our approach with $\theta = 1$, and UOT with $\lambda = \gamma \varepsilon$ for $\gamma \in \{0.1, 1, 10, 100\}$, with all remaining hyperparameters selected as described above. Each plan is evaluated via $\overline{R_k}$ as defined in~\eqref{eq:recall_k}, and results are averaged over $10$ independent replicates.
    
    \begin{algorithm}[H]
    \caption{Bootstrap--Lepski Bandwidth Selection}
    \label{alg:bootstrap_lepski}

    \textbf{Input:} Datasets $\mathbf{X} = \{\mathbf{x}_i\}_{i=1}^n \subset \mathbb{R}^D$; number of bootstrap resamples $k \in \mathbb{N}$; bandwidth range $(h_{\min}, h_{\max})$; number of candidate bandwidths $m \in \mathbb{N}$; quantile level $\alpha \in (0,1)$; threshold scaling constant $C > 0$.

    \begin{algorithmic}[1]

        \State \textbf{Construct bandwidth grid.}
        Form the geometric grid
        $\mathcal{H} = \{h_1, \ldots, h_m\}$ with
        \begin{equation}\label{eq:bw_grid}
            h_j = h_{\min} \cdot \gamma^{\,j-1}, \qquad j \in [m],
        \end{equation}
        where $\gamma = (h_{\max}/h_{\min})^{1/(m-1)}$. 

        \State \textbf{Bootstrap Resampling.}
        Resample $k$ i.i.d.\ datasets $\widetilde{\mathbf{X}}^{(1)}, \ldots, \widetilde{\mathbf{X}}^{(k)}$ of size $n$ uniformly with replacement from $\mathbf{X}$.

        \State \textbf{Compute pairwise squared distance matrices.}
        \begin{itemize}
            \item Self-distances on $\mathbf{X}$:
                \begin{equation}\label{eq:self_dist}
                    D_{ij} = \|\mathbf{x}_i - \mathbf{x}_j\|^2,
                    \qquad \forall i,j \in [n].
                \end{equation}
            \item Cross-distances between $\mathbf{X}$ and each resample $\widetilde{\mathbf{X}}^{(\ell)}$:
                \begin{equation}\label{eq:cross_dist}
                    D_{ij}^{(\ell)} = \bigl\|\mathbf{x}_i - \tilde{\mathbf{x}}_j^{(\ell)}\bigr\|^2,
                    \qquad \forall i,j \in [n], \, \ell \in [k].
                \end{equation}
        \end{itemize}

        \State \textbf{Evaluate mean-normalized density estimates.}
        For each $h \in \mathcal{H}$, compute the mean-normalized KDE (Definition~\ref{def:self_normalized_kde}) at each $\mathbf{x}_p \in \mathbf{X}$ as:
        \begin{equation}\label{eq:scaled_density}
            \hat{f}_h(\mathbf{x}_p)
            = \frac{\displaystyle \sum_{j=1}^{n} \exp\!\left(-\dfrac{D_{pj}}{2h^2}\right)}
                   {\displaystyle \dfrac{1}{n}\sum_{i=1}^{n}\sum_{j=1}^{n}
                    \exp\!\left(-\dfrac{D_{ij}}{2h^2}\right)},
        \end{equation}
        and its bootstrap counterpart on $\widetilde{\mathbf{X}}^{(\ell)}$,
        with normalizer retained from \eqref{eq:scaled_density} as:
        \begin{equation}\label{eq:boot_density}
            \hat{f}_h^{(\ell)}(\mathbf{x}_p)
            = \frac{\displaystyle \sum_{j=1}^{n}
                    \exp\!\left(-\dfrac{D_{pj}^{(\ell)}}{2h^2}\right)}
                   {\displaystyle \dfrac{1}{n}\sum_{i=1}^{n}\sum_{j=1}^{n}
                    \exp\!\left(-\dfrac{D_{ij}}{2h^2}\right)}, \qquad \forall l \in [k].
        \end{equation}

        \State \textbf{Form bootstrap null distributions and pointwise thresholds.}
        For each pair $h_i < h_j$ and each point $\mathbf{x}_p \in \mathbf{X}$,
        compute the centered absolute bootstrap differences:
        \begin{equation}\label{eq:boot_diff}
            \widetilde{V}_{ij}^{(\ell)}(\mathbf{x}_p)
            = \Bigl|\bigl(\hat{f}_{h_j}^{(\ell)}(\mathbf{x}_p)
                         - \hat{f}_{h_i}^{(\ell)}(\mathbf{x}_p)\bigr)
                   - \bar{V}_{ij}(\mathbf{x}_p)\Bigr|,
            \quad \ell \in [k],
        \end{equation}
        where $\bar{V}_{ij}(\mathbf{x}_p) = \frac{1}{k}\sum_{\ell=1}^k
        \bigl(\hat{f}_{h_j}^{(\ell)}(\mathbf{x}_p) - \hat{f}_{h_i}^{(\ell)}(\mathbf{x}_p)\bigr)$.
        Set the pointwise threshold as the scaled $(1-\alpha)$-quantile:
        \begin{equation}\label{eq:threshold}
            T_{ij}(\mathbf{x}_p)
            = C \cdot q_{1-\alpha}\!\Bigl(
                \widetilde{V}_{ij}^{(1)}(\mathbf{x}_p),\, \ldots,\,
                \widetilde{V}_{ij}^{(k)}(\mathbf{x}_p)
              \Bigr).
        \end{equation}

        \State \textbf{Lepski bandwidth selection.}
        Sweep $j = 1, \ldots, m$ in increasing order and return the largest $h_j$ such that
        for all $i < j$,
        \begin{equation}\label{eq:lepski_condition}
            \sum_{p=1}^n \bigl|\hat{f}_{h_j}(\mathbf{x}_p) - \hat{f}_{h_i}(\mathbf{x}_p)\bigr|
            \;\le\;
            \sum_{p=1}^n T_{ij}(\mathbf{x}_p).
        \end{equation}
    \end{algorithmic}
    \end{algorithm}

    \section{Proofs and Related Background}

    \subsection{Proof of Proposition~\ref{prop:density_invariant_eot}}\label{sec:proof_opt}
    \begin{proof}
        The proof proceeds in two steps. We first verify that the feasible set of~\eqref{eq:density_invariant_opt} is nonempty and compact, which, together with the strict convexity and continuity of the objective, guarantees the existence of a unique minimizer. We then show that the matrix $\mathbf{W}^{(\theta)}$ defined in~\eqref{eq:EOT_solution_form} is feasible and satisfies the first-order optimality conditions, which are sufficient to identify it as the unique minimizer.
    
        For notational simplicity, we define
        \begin{equation}\label{eq:notation}
        r_i \;:=\; \frac{n}{\sqrt{S^{(\theta)}}\,(\hat{f}_i)^{\,\theta}}, \qquad c_j \;:=\; \frac{m\sqrt{S^{(\theta)}}}{(\hat{g}_j)^{\,\theta}}, \qquad C_{ij} \;:=\; \|\mathbf{x}_i - \mathbf{y}_j\|_2^{2}, \qquad h_{ij} \;:=\; (\hat{f}_i)^{\,\theta}(\hat{g}_j)^{\,\theta}, \qquad \forall i\in[m],\ \forall j\in[n].
        \end{equation}
        
        \textbf{Step 1.} A direct calculation using the definition of $S^{(\theta)}$ in~\eqref{eq:S_scale} yields $\sum_{i=1}^{m} r_i = \sum_{j=1}^{n} c_j$, so the feasible set of~\eqref{eq:density_invariant_opt} is nonempty: for instance, the rank-one plan $W_{ij} = h_{ij}\,r_i c_j / \sum_{k} r_k$ satisfies both marginal constraints. Moreover, the feasible set is closed (as the intersection of finitely many hyperplanes with the nonnegative orthant) and bounded (since each entry satisfies $0 \le W_{ij} \le h_{ij}\min\{n r_i,\, m c_j\}$), hence compact.
    
        \textbf{Step 2.} \emph{Feasibility.} By~\eqref{eq:EOT_solution_form} and the definition of $\mathbf{M}^{(\theta)}$ in~\eqref{eq:M}, we have
        \begin{equation}
            \frac{W_{ij}^{(\theta)}}{h_{ij}} \;=\; \alpha_i^{(\theta)}\, M_{ij}^{(\theta)}\,\beta_j^{(\theta)}, \qquad \forall i \in [m],\ j \in [n].
        \end{equation}
        Substituting into the Sinkhorn scaling conditions in~\eqref{eq:row_col_constraint} gives exactly the marginal constraints in~\eqref{eq:density_invariant_opt}.
    
        \emph{First-order conditions.} 
        Introducing multipliers $\boldsymbol{\lambda}\in\mathbb{R}^{m}$ and $\boldsymbol{\mu}\in\mathbb{R}^{n}$ for the row and column constraints in~\eqref{eq:density_invariant_opt} respectively, the Lagrangian reads
        \[
        \mathcal{L}(\mathbf{\tilde{W}},\boldsymbol{\lambda},\boldsymbol{\mu}) \;=\; \sum_{i,j}\frac{\tilde{W}_{ij}}{h_{ij}}\,C_{ij} \;+\; \varepsilon\sum_{i,j}\frac{\tilde{W}_{ij}}{h_{ij}}\log\tilde{W}_{ij} \;-\; \sum_{i}\lambda_i\!\left(\frac{1}{n}\sum_{j}\frac{\tilde{W}_{ij}}{h_{ij}} - \frac{1}{\sqrt{S^{(\theta)}}\,\hat{f}_i^{\,\theta}}\right) \;-\; \sum_{j}\mu_j\!\left(\frac{1}{m}\sum_{i}\frac{\tilde{W}_{ij}}{h_{ij}} - \frac{\sqrt{S^{(\theta)}}}{(\hat{g}_j)^{\,\theta}}\right).
        \]
        Differentiating with respect to $\tilde{W}_{ij}$, and setting $\partial\mathcal{L}/\partial \tilde{W}_{ij} = 0$ yields

        \begin{equation}
            \tilde{W}_{ij} \;=\; \exp\!\left(\frac{\lambda_i}{n\varepsilon} - \frac{1}{2} \right)\cdot \exp\!\left(-\frac{C_{ij}}{\varepsilon}\right)\cdot \exp\!\left(\frac{\mu_j}{m\varepsilon} - \frac{1}{2} \right), \quad \forall i\in [m], \forall j \in [n].
        \end{equation}
        We note $\mathbf{W}^{(\theta)}$ in~\eqref{eq:EOT_solution_form} factors in the same way.
    \end{proof}

    \subsection{Auxiliary Definition and Lemma for the Proof of Theorem~\ref{thm:finite_population_density_known} and~\ref{thm:finite_population_density_estimated}}

    We first introduce the matrix scaling problem and state the relevant lemmas, following Section SM5 of~\cite{boris_EOT}.

    Matrix scaling is the process of diagonally scaling a positive matrix $\mathbf{A} \in \mathbb{R}_{++}^{m \times n}$ to achieve certain prescribed row sums $\mathbf{r} \in \mathbb{R}_{++}^{m}$ and column sums $\mathbf{c} \in \mathbb{R}_{++}^{n}$, where $\|\mathbf{r}\|_1 = \|\mathbf{c}\|_1$. We say a pair of vectors $\Alpha \in \mathbb{R}_{++}^m$ and $\Beta \in \mathbb{R}_{++}^n$ scales $\mathbf{A}$ to row sums $\mathbf{r}$ and column sums $\mathbf{c}$, if
    \begin{equation}\label{eq:scaling_def}
        r_i = \sum_{j=1}^{n} \alpha_i A_{i,j} \beta_j, \quad \text{ and } \quad c_j =  \sum_{i=1}^{m} \alpha_i A_{i,j} \beta_j,
    \end{equation}
    for all $i \in [m]$ and $j \in [n]$. We refer to $\Alpha$ and $\Beta$ from~\eqref{eq:scaling_def} (or their entries) as \textit{scaling factors} of $\mathbf{A}$.

    \begin{lemma}[Existence and uniqueness~\cite{sinkhorn_existence}]
    \label{lemma:existence_uniqueness_scaling}
    Suppose that $\mathbf{A} \in \mathbb{R}_{++}^{m\times n}$, $\mathbf{r} \in \mathbb{R}_{++}^{m}$, and $\mathbf{c} \in \mathbb{R}_{++}^{n}$ with $\|\mathbf{r}\|_1 = \|\mathbf{c}\|_1$. Then, there exists a pair of vectors $\Alpha \in \mathbb{R}_{++}^m$ and $\Beta \in \mathbb{R}_{++}^n$ that scales $\mathbf{A}$ to row sums $\mathbf{r}$ and column sums $\mathbf{c}$. Furthermore, the pair of scaling factor $(\Alpha, \Beta)$ is unique up to a scaler multiple; that is, any other valid
    scaling pair is of the form $(c\, \Alpha,\, c^{-1} \Beta)$ for some constant $c > 0$.
    \end{lemma}

    The following lemma characterizes the stability of the scaling factors to perturbations in the row and column sums. In particular, it states that if there exists a pair of positive vectors $(\hat{\Alpha}, \hat{\Beta})$ that approximately scales $\mathbf{A}$ to certain row and column sums, then $(\hat{\Alpha}, \hat{\Beta})$ must be close to a true pair of scaling factors for these row and column sums.
    
    \begin{lemma}[Stability of scaling factors under approximate scaling~\cite{boris_EOT}]
    \label{lem:scaling_factor_concen}
    Let $\mathbf{A} \in \mathbb{R}_{++}^{m \times n}$ and denote $a = \min_{i,j} A_{i,j}$, $b = \max_{i,j} A_{i,j}$. Suppose that there exists $\delta \in (0,1)$ and vectors $\hat{\Alpha} \in \mathbb{R}_{++}^{m}$ and $\hat{\Beta} \in \mathbb{R}_{++}^{n}$, such that
    \begin{equation}
    \left| \frac{1}{c_j} \sum_{i=1}^{m} \hat{\alpha}_i A_{i,j} \hat{\beta}_j - 1 \right| \leq \delta, \qquad \left| \frac{1}{r_i} \sum_{j=1}^{n} \hat{\alpha}_i A_{i,j} \hat{\beta}_j - 1 \right| \leq \delta,
    \end{equation}
    for all $i \in [m]$ and $j \in [n]$. Then, there exists a pair of vectors $\Alpha \in \mathbb{R}_{++}^{m}$ and $\Beta \in \mathbb{R}_{++}^{n}$ that scales $\mathbf{A}$ to row sums $\mathbf{r}$ and column sums $\mathbf{c}$, satisfying
    \begin{equation}\label{eq:lemma_alpha_bound}
    \frac{|\alpha_i - \hat{\alpha}_i|}{\hat{\alpha}_i} \leq \delta \left( \frac{1}{1-\delta} + \frac{4\sqrt{b}}{a^2 \sqrt{s} C_1^{3/2} C_2^{3/2} m \min_i r_i} \right),
    \end{equation}
    \begin{equation}\label{eq:lemma_beta_bound}
    \frac{|\beta_j - \hat{\beta}_j|}{\hat{\beta}_j} \leq \delta \left( \frac{1}{1-\delta} + \frac{4\sqrt{b}}{a^2 \sqrt{s} C_1^{3/2} C_2^{3/2} n \min_j c_j} \right),
    \end{equation}
    for all $i \in [m]$ and $j \in [n]$, where $C_1 = \min_i \{ \hat{\alpha}_i / r_i \}$, $C_2 = \min_j \{ \hat{\beta}_j/ c_j \}$, and $s = \|\mathbf{r}\|_1 = \|\mathbf{c}\|_1$.
    \end{lemma}

    For notational simplicity, we introduce the notation of \textit{order in probability}, denoted as $\widetilde{\mathcal{O}}_{m,n}^{(\varepsilon)}$.
    \begin{definition} \label{def:order_prob}
        For a random variable $z$, we say $z = \widetilde{\mathcal{O}}_{m,n}^{(\varepsilon)}\left(f(m,n)\right)$ if there exist global constant \sloppy $t_0, n_0(\varepsilon, t), C(\varepsilon) >0$ such that for all $t>t_0$, $n \geq n_0(\varepsilon, t)$, and $m 
        \geq n^{\gamma}$ (under Assumption~\ref{assump:m_and_n}),
            \begin{equation}
                |z| \leq {t} C(\varepsilon) f(m,n),
            \end{equation}
            with probability at least $1-n^{-t}$.  When $n_0$ and $C$ can be chosen independently of $\varepsilon$, we write $z = \widetilde{\mathcal{O}}_{m,n}\!\left(f(m,n)\right)$.
    \end{definition}
 
    Note that under Definition~\ref{def:order_prob}, given a polynomial function $P(n)$, if for any $ i = 1, 2, \ldots,P(n)$, $z_i = \widetilde{\mathcal{O}}^{(\varepsilon)}_{m,n}\left(f(m,n)\right)$, then \begin{equation}\label{eq:order_max}
            \max_{i=1,2\ldots,P(n)} |z_i| = \widetilde{\mathcal{O}}_{m,n}^{(\varepsilon)}\left(f(m,n)\right), 
        \end{equation}
    which can be derived by applying the union bound.

    Furthermore, if $Z = \widetilde{\mathcal{O}}_{m,n}^{(\varepsilon)}(f(m,n))$ and $Y = g^{(\varepsilon)}(Z)$, where $g^{(\varepsilon)}$ is $\mathcal{C}^1$ on a neighborhood of $0$ for all $\varepsilon > 0$ and \sloppy$\lim_{m,n \to \infty} f(m,n) = 0$, then by a Taylor expansion of $g^{(\varepsilon)}(x)$ around zero,
    \begin{equation}\label{eq:order_taylor}
        Y = g^{(\varepsilon)}(0) + \widetilde{\mathcal{O}}_{m,n}^{(\varepsilon)}(f(m,n)).
    \end{equation}
    We will use properties \eqref{eq:order_max} and \eqref{eq:order_taylor} of Definition~\ref{def:order_prob} 
    seamlessly throughout the remaining proofs.

\subsection{Proof of Theorem~\ref{thm:finite_population_density_known}}\label{sec:proof density known}

For notation simplicity, we define:
\begin{equation}
    Z_x^{(\theta)} = \int_{\mathcal{M}_x} \mu(\mathbf{x})^{1-\theta}\, d\omega(\mathbf{x}),
    \qquad
    Z_y^{(\theta)} = \int_{\mathcal{M}_y} \nu(\mathbf{y})^{1-\theta}\, d\eta(\mathbf{y}).
\end{equation}

\begin{proof}
    Let $\Alpha^{(\theta)}$ and $\Beta^{(\theta)}$ be the scaling factors in Algorithm~\ref{alg:EOT_invariant} that diagonally scale $\mathbf{M}^{(\theta)}$, defined in~\eqref{eq:M}, to satisfy the row and column constraints in~\eqref{eq:row_col_constraint}, for fixed $\theta \in [0,1]$, $\hat{\mathbf{f}} = \boldsymbol{\mu}$, and $\hat{\mathbf{g}} = \boldsymbol{\nu}$. Define approximate scaling factors $\hat{\Alpha}$ and $\hat{\Beta}$ as
    \begin{equation}\label{eq:hat_scaling_factor}
        \hat{\alpha}_i^{(\theta)} = \frac{1}{\sqrt{\Zx}}u^{(\theta)}_\varepsilon(\mathbf{x}_i),\quad \text{ and } \quad \hat{\beta}_j^{(\theta)} = \frac{1}{\sqrt{\Zy}} v^{(\theta)}_\varepsilon(\mathbf{y}_j),
    \end{equation}
    for any $i \in [m]$ and any $j \in [n]$, where $u^{(\theta)} _\varepsilon$ and  $v^{(\theta)}_\varepsilon$ are defined in~\eqref{eq:integral_eq}.

    We first derive a bound on $S^{(\theta)}$ in~\eqref{eq:S_scale}.
    Applying Hoeffding's inequality, we have:
    \begin{equation}\label{eq:S1}
        \frac{1}{m} \sum_{i=1}^m \frac{1}{(\hat{f}_i)^{\theta}}  
        = \frac{1}{m} \sum_{i=1}^m \frac{1}{(\mu(\mathbf{x}_i))^{\theta}}
        = \int_{\mathcal{M}_x} 
         \frac{1}{(\mu(\mathbf{x}))^{\theta}} \mu(\mathbf{x}) \, d\omega(\mathbf{x}) + \widetilde{\mathcal{O}}_{m}\left(\sqrt{\frac{\log m}{m}}\right)
         = \Zx +  \widetilde{\mathcal{O}}_{m}\left(\sqrt{\frac{\log m}{m}}\right),
    \end{equation}
    where $\Zx$ is defined in~\eqref{eq:tilt_density}. Similarly, 
    \begin{equation}\label{eq:S2}
        \frac{1}{n} \sum_{j=1}^n \frac{1}{(\hat{g}_j)^{\theta}}  
        = \Zy +  \widetilde{\mathcal{O}}_{n}\left(\sqrt{\frac{\log n}{n}}\right) 
        = \Zy +  \widetilde{\mathcal{O}}_{m}\left(\sqrt{\frac{\log m}{m}}\right).
    \end{equation} 
    In the derivation, we use the fact that $\sqrt{\log n/n} \leq 2\sqrt{\log m/m}$ for all sufficiently large $m$ under Assumption~\ref{assump:m_and_n}.
    
    Combining~\eqref{eq:S1} and~\eqref{eq:S2},
    \begin{equation}\label{eq:S_bound}
        S^{(\theta)} = \frac{\Zx}{\Zy} + \widetilde{\mathcal{O}}_{m}\left(\sqrt{\frac{\log m}{m}}\right), \quad \text{ and } \quad
        \frac{1}{S^{(\theta)}} = \frac{\Zy}{\Zx} + \widetilde{\mathcal{O}}_{m}\left(\sqrt{\frac{\log m}{m}}\right).
    \end{equation}
    
    We next show that $\hat{\Alpha}^{(\theta)}$ and $\hat{\Beta}^{(\theta)}$ approximately scale $\mathbf{M}^{(\theta)}$ to have prescribed row sums $\mathbf{r}$ and column sums $\mathbf{c}$. Recall that, from~\eqref{eq:row_col_constraint},
    \begin{equation}\label{eq:row_col_sum}
        r_i = \frac{n}{(\hat{f}_i)^{\theta}\sqrt{S^{(\theta)}}} 
        = \frac{n}{(\mu(\mathbf{x}_i))^{\theta} \cdot \sqrt{S^{(\theta)}}}, \quad \text{and} \quad
        c_j = \frac{m\sqrt{S^{(\theta)}}}{(\hat{g}_j)^{\theta}} = \frac{m\sqrt{S^{(\theta)}}}{(\nu(\mathbf{y}_j))^{\theta}},
    \end{equation}
    for all $ i \in [m]$ and $j \in [n]$, and $S^{(\theta)}$ is defined in~\eqref{eq:S_scale}.
    Using~\eqref{eq:hat_scaling_factor},~\eqref{eq:row_col_sum}, $\mathcal{K}_{\varepsilon}$ and $\mathbf{M}^{(\theta)}$ defined in~\eqref{eq:W_kernel} and~\eqref{eq:M}, we have: for any $i \in [m]$,
    \begin{equation}\label{eq:row_discrete}
    \begin{aligned}
        \frac{1}{r_i} \sum_{j=1}^n \hat{\alpha}_i^{(\theta)} M_{ij}^{(\theta)} \hat{\beta}_j^{(\theta)} 
        &= \frac{(\mu(\mathbf{x}_i))^{\theta} \cdot \sqrt{S^{(\theta)}}}{n} \sum_{j=1}^n \frac{u^{(\theta)} _\varepsilon(\mathbf{x}_i) }{\sqrt{\Zx}}\frac{\mathcal{K}_\varepsilon(\mathbf{x}_i, \mathbf{y}_j)}{(\mu(\mathbf{x}_i)\nu(\mathbf{y}_j))^{\theta}} \frac{v^{(\theta)} _\varepsilon(\mathbf{y}_j)}{\sqrt{\Zy}} \\
        &= \frac{\sqrt{S^{(\theta)} }}{\sqrt{\Zx\Zy}} \frac{1}{n} \sum_{j=1}^n u^{(\theta)} _\varepsilon(\mathbf{x}_i) \mathcal{K}_\varepsilon(\mathbf{x}_i, \mathbf{y}_j) v^{(\theta)} _\varepsilon(\mathbf{y}_j) \frac{1}{(\nu(\mathbf{y}_j))^{\theta}}.
    \end{aligned}
    \end{equation}
    Conditioning on $\mathbf{x}_i$ and applying Hoeffding's inequality yields
    \begin{equation}\label{eq:row_bound}
    \begin{aligned}
        &\frac{1}{r_i} \sum_{j=1}^n \hat{\alpha}_i^{(\theta)}M_{ij}^{(\theta)} \hat{\beta}_j^{(\theta)}
        = \frac{\sqrt{S^{(\theta)} }}{\sqrt{\Zx\Zy}} \int_{\mathcal{M}_y} u^{(\theta)} _\varepsilon(\mathbf{x}_i) \mathcal{K}_\varepsilon(\mathbf{x}_i,\mathbf{y}) v^{(\theta)} _\varepsilon(\mathbf{y}) \cdot (\nu(\mathbf{y}))^{1-\theta}\, d\eta(\mathbf{y}) + \widetilde{\mathcal{O}}_{n}\left(\sqrt{\frac{\log n}{n}}\right) \\
        &= \frac{\sqrt{S^{(\theta)}}\Zy}{\sqrt{\Zx\Zy}} +  \widetilde{\mathcal{O}}_{n}\left(\sqrt{\frac{\log n}{n}}\right) = \sqrt{\frac{\Zx}{\Zy} + \widetilde{\mathcal{O}}_{m}\left(\sqrt{\frac{\log m}{m}}\right)} \sqrt{\frac{\Zy}{\Zx}} + \widetilde{\mathcal{O}}_{n}\left(\sqrt{\frac{\log n}{n}}\right)\\
        &= 1 +  \widetilde{\mathcal{O}}_{m}\left(\sqrt{\frac{\log m}{m}}\right) + \widetilde{\mathcal{O}}_{n}\left(\sqrt{\frac{\log n}{n}}\right) = 1 + \widetilde{\mathcal{O}}_{m}\left(\sqrt{\frac{\log m}{m}}\right),
    \end{aligned}
    \end{equation}
    where we use~\eqref{eq:integral_eq} and~\eqref{eq:S_bound} in the second line. Under Assumption~\ref{assump:m_and_n}, a union bound over all $i \in [m]$ shows that~\eqref{eq:row_bound} holds simultaneously for all $i \in [m]$.
    
    Similarly, for any $j \in [n]$,
    \begin{equation}
    \begin{aligned}
        \frac{1}{c_j} \sum_{i=1}^m \hat{\alpha}_i^{(\theta)} M_{ij}^{(\theta)} \hat{\beta}_j^{(\theta)} 
        &= \frac{1}{\sqrt{S^{(\theta)}}
        \sqrt{\Zx\Zy}} \frac{1}{m} \sum_{i=1}^m u^{(\theta)} _\varepsilon(\mathbf{x}_i) \mathcal{K}_\varepsilon(\mathbf{x}_i, \mathbf{y}_j) v^{(\theta)} _\varepsilon(\mathbf{y}_j) \frac{1}{(\mu(\mathbf{x}_i))^{\theta}}.
    \end{aligned}
    \end{equation}
    Following an analogous derivation of~\eqref{eq:row_bound},
    \begin{equation}\label{eq:col_bound}
    \begin{aligned}
        \frac{1}{c_j} \sum_{i=1}^m \hat{\alpha}_i^{(\theta)} M_{ij}^{(\theta)} \hat{\beta}_j^{(\theta)} & 
        = 1 + \widetilde{\mathcal{O}}_{m}\left(\sqrt{\frac{\log m}{m}}\right),
    \end{aligned}
    \end{equation}
    for all $j \in [n]$.
    
    We now apply Lemma~\ref{lem:scaling_factor_concen} to the approximate scaling factor of $\mathbf{M}^{(\theta)}$, namely $\hat{\Alpha}^{(\theta)}$ and $\hat{\Beta}^{(\theta)}$. We first evaluate and bound the different quantities appearing in Lemma~\ref{lem:scaling_factor_concen}. 

    By the definition of $s$ in Lemma~\ref{lem:scaling_factor_concen} and using~\eqref{eq:S1},~\eqref{eq:S_bound}, and~\eqref{eq:row_col_sum}:
    \begin{equation}\label{eq:s}
        s = \|\mathbf{r}\|_1 
        = \sum_{i=1}^m \frac{n}{(\mu(\mathbf{x}_i))^{\theta} \cdot \sqrt{S^{(\theta)}}} 
        = \frac{nm}{\sqrt{S^{(\theta)} }} \cdot \frac{1}{m}\sum_{i=1}^m \frac{1}{(\mu(\mathbf{\mathbf{x}_i}))^{\theta}} 
        = nm \left(\sqrt{\Zx\Zy} +  \widetilde{\mathcal{O}}_{m}\left(\sqrt{\frac{\log m}{m}}\right)\right).
    \end{equation}

    Since $u^{(\theta)} _\varepsilon$, $v^{(\theta)} _\varepsilon$, $\mu$, and $\nu$ are positive and continuous on the compact manifolds, they are bounded away from zero and infinity. Thus, there exist some positive constants $C^{'}_{1}(\varepsilon)$ and $C^{'}_{2}(\varepsilon)$ (which may depend on $\varepsilon$), such that $C_1$ and $C_2$ defined in Lemma~\ref{lem:scaling_factor_concen} satisfy:
    \begin{equation}\label{eq:lemma_C}
        C_1 = \min_i \left\{\frac{\hat{\alpha}_i^{(\theta)} }{r_i}\right\} 
        \geq \frac{C^{'}_{1}(\varepsilon)}{n}, \quad  C_2 = \min_j \left\{\frac{\hat{\beta}_j^{(\theta)}}{c_j}\right\}
        \geq \frac{C^{'}_{2}(\varepsilon)}{m}.
    \end{equation}
    Similarly, there exist positive constants $C^{'}_{3}(\varepsilon)$ and $C^{'}_{4}(\varepsilon)$ such that
\begin{equation}\label{eq:lemma_M}
   C^{'}_{3}(\varepsilon) \leq M_{ij}^{(\theta)} = \frac{\operatorname{exp}\left(-\|\mathbf{x}_i - \mathbf{y}_j\|_2^2/\varepsilon\right)}{(\mu(\mathbf{x}_i)\nu(\mathbf{y}_j))^{\theta}}\leq C^{'}_{4}(\varepsilon).
\end{equation}
Thus, $a$ and $b$ defined in Lemma~\ref{lem:scaling_factor_concen} satisfy
\begin{equation}\label{eq:ab}
    C^{'}_{3}(\varepsilon) \leq a \leq b \leq C^{'}_{4}(\varepsilon).
\end{equation}

Hence, for any $\delta \in (0,1/2]$, using~\eqref{eq:s},~\eqref{eq:lemma_C}, and~\eqref{eq:ab}, the coefficients of~\eqref{eq:lemma_alpha_bound} and~\eqref{eq:lemma_beta_bound} are bounded:
\begin{equation}\label{eq:constant}
    \left(\frac{1}{1-\delta}+\frac{4\sqrt{b}}{a^2 \sqrt{s} C_1^{3/2} C_2^{3/2} m \min_i r_i}\right) \leq D_1(\varepsilon), \text{ and } \left(\frac{1}{1-\delta}+\frac{4\sqrt{b}}{a^2 \sqrt{s} C_1^{3/2} C_2^{3/2} n \min_j c_j}\right) \leq D_2(\varepsilon) ,
\end{equation}
for some positive constants $D_1(\varepsilon)$ and $D_2(\varepsilon)$.

Applying Lemma~\ref{lem:scaling_factor_concen} with $\delta := \tau\, C(\varepsilon) \sqrt{\log m / m}$, where
$C(\varepsilon)$ is the constant associated with~\eqref{eq:row_bound}
and~\eqref{eq:col_bound} in Definition~\ref{def:order_prob} and $\delta \le 1/2$
upon choosing $n_0(\varepsilon, \tau)$ properly, there exists a pair of positive vectors $(\Alpha^{(\theta)'}, \Beta^{(\theta)'})$ that scales $\mathbf{M^{(\theta)}}$ to row sums $\mathbf{r}$ and column sums $\mathbf{c}$, such that
\begin{equation}
    \frac{|\alpha_i^{(\theta)'} - \hat{\alpha}_i^{(\theta)}|}{\hat{\alpha}_i^{(\theta)}} = \widetilde{\mathcal{O}}_{m}^{(\varepsilon)}\left(\sqrt{\frac{\log m}{m}}\right), \quad \text{ and } \quad \frac{|\beta_j^{(\theta)'} - \hat{\beta}_j^{(\theta)}|}{\hat{\beta}_j^{(\theta)}} =\widetilde{\mathcal{O}}_{m}^{(\varepsilon)}\left(\sqrt{\frac{\log m}{m}}\right),
\end{equation}
for any $i \in [m]$ and any $j \in [n]$.

By Lemma~\ref{lemma:existence_uniqueness_scaling} and the definition of scaling factors, the two pairs of scaling factors ($\Alpha^{(\theta)}$, $\Beta^{(\theta)}$) and ($\Alpha^{(\theta)'}$, $\Beta^{(\theta)'}$) satisfy $\alpha_i^{(\theta)}\beta_j^{(\theta)} = \alpha_i^{(\theta)'}\beta_j^{(\theta)'}$, for any $i \in [m]$ and $j \in [n]$. Hence, applying a union bound over all $i \in [m]$ and $j \in [n]$ under Assumption~\ref{assump:m_and_n}, we have: for any $i \in [m]$ and $j \in [n]$,
\begin{equation}\label{eq:mul1}
\begin{aligned}
    &W_{ij}^{(\theta)} = \alpha_i^{(\theta)} e^{\frac{-\|\mathbf{x}_i - \mathbf{y}_j\|_2^2}{\varepsilon}} \beta_j^{(\theta)} 
    = \alpha_i^{(\theta)'} e^{\frac{-\|\mathbf{x}_i - \mathbf{y}_j\|_2^2}{\varepsilon}}\beta_j^{(\theta)'} 
    = \hat{\alpha}_i^{(\theta)} e^{\frac{-\|\mathbf{x}_i - \mathbf{y}_j\|_2^2}{\varepsilon}}\hat{\beta}_j^{(\theta)}\left(1+\widetilde{\mathcal{O}}_{m}^{(\varepsilon)}\left(\sqrt{\frac{\log m}{m}}\right)\right) \\
    &= \frac{u^{(\theta)} _\varepsilon(\mathbf{x}_i) \mathcal{K}_\varepsilon(\mathbf{x}_i,\mathbf{y}_j) v^{(\theta)} _\varepsilon(\mathbf{y}_j) }{\sqrt{\Zx\Zy}} \left(1+\widetilde{\mathcal{O}}_{m}^{(\varepsilon)}\left(\sqrt{\frac{\log m}{m}}\right)\right) = \frac{\mathcal{W}^{(\theta)}_{\varepsilon}(\mathbf{x}_i,\mathbf{y}_j)}{\sqrt{\Zx\Zy}} \left(1+\widetilde{\mathcal{O}}_{m}^{(\varepsilon)}\left(\sqrt{\frac{\log m}{m}}\right)\right)
\end{aligned}
\end{equation}
Since $\mathcal{W}_\varepsilon^{(\theta)}$ is continuous and bounded above on the
compact product manifold $\mathcal{M}_x \times \mathcal{M}_y$, the multiplicative
estimate in~\eqref{eq:mul1} implies the additive bound
in~\eqref{eq:W_ij2W_kernel}, completing the proof.

\end{proof}

\subsection{Auxiliary Lemmas for the Proof of Lemma~\ref{lem:self-normalizing kde}}\label{sec:lemma_kde}

\begin{lemma}
\label{lem:bias}
Let $\mathcal{M} \subset \mathbb{R}^{D}$ be a smooth, compact, boundaryless
$d$-dimensional Riemannian manifold, and let $f \in \mathcal{C}^3(\mathcal{M})$.
Then there exist $h_0(\mathcal{M}) > 0$ and
$C = C(\mathcal{M}, f) > 0$ such that for all
$0 < h \le h_0(\mathcal{M})$,
\begin{equation}
    \sup_{x \in \mathcal{M}} \left|
    \frac{1}{(\pi h)^{d/2}} \int_{\mathcal{M}}
    \exp\!\left(-\frac{\|x - y\|^2}{h}\right) f(y)\, dV(y)
    - f(x) \right| \;\le\; Ch .
\end{equation}
\end{lemma}

\begin{proof}
By \cite[Lemma 8]{diffusion_maps} applied to the Gaussian kernel
$h(t) = e^{-t}$, for which
$\int_{\mathbb{R}^d} h(\|u\|^2)\,du = \pi^{d/2}$ and
$\int_{\mathbb{R}^d} u_1^2 h(\|u\|^2)\,du = \tfrac{1}{2}\pi^{d/2}$, we have
\begin{equation*}
    \frac{1}{(\pi h)^{d/2}} \int_{\mathcal{M}}
    \exp\!\left(-\frac{\|x - y\|^2}{h}\right) f(y)\, dV(y)
    = f(x) + \frac{h}{4}\left[\omega(x)f(x)
    + \Delta_{\mathcal{M}}f(x)\right] + O(h^2),
\end{equation*}
uniformly in $x \in \mathcal{M}$, where $\omega$ is a potential term depending
on the embedding of $\mathcal{M}$ and $\Delta_{\mathcal{M}}$ is the Laplace–Beltrami operator. (The estimate in \cite{diffusion_maps} is
stated uniformly over points at distance more than $\varepsilon^{\gamma}$ from
$\partial\mathcal{M}$ for any fixed $0 < \gamma < 1/2$; since
$\partial\mathcal{M} = \emptyset$ this is all of $\mathcal{M}$.) Since $\omega$, $f$ and
$\Delta_{\mathcal{M}}f$ are continuous on the compact manifold $\mathcal{M}$,
the first-order term is bounded by a constant multiple of $h$
uniformly in $x$, and the claim follows.
\end{proof}

\begin{lemma}
\label{lem:kde-uniform}
Let $\mathcal{M} \subset \mathbb{R}^{D}$ be a smooth, compact, boundaryless
$d$-dimensional Riemannian manifold isometrically embedded in $\mathbb{R}^{D}$,
let $\mu \in \mathcal{C}^{3}(\mathcal{M})$ be a positive probability density on
$\mathcal{M}$, and let $\mathbf{x}_1, \dots, \mathbf{x}_n \overset{\mathrm{iid}}{\sim} \mu$.
Define
\begin{equation}\label{eq:kde1}
    \hat{p}(x) \;:=\; \frac{1}{n}\sum_{j=1}^{n} \frac{1}{(\pi h)^{d/2}}
    \exp\!\left(-\frac{\|x - \mathbf{x}_j\|^{2}}{h}\right).
\end{equation}
Then there exist $h_0(\mathcal{M}, \mu) > 0$ and $C = C(\mathcal{M}, \mu) > 0$
such that for all $n \ge 2 h_0(\mathcal{M},\mu)^{-1/2}$ and all fixed $h$ with
$4 n^{-2} \le h \le h_0(\mathcal{M}, \mu)$, with probability at least
$1 - n^{-2}$,
\begin{equation}
\label{eq:kde-uniform}
    \sup_{x \in \mathcal{M}} \bigl| \hat{p}(x) - \mu(x) \bigr|
    \;\le\; C\left( h + \sqrt{\frac{\log n}{n h^{d}}} \right).
\end{equation}
\end{lemma}

\begin{proof}
We decompose the error via the triangle inequality:
\begin{equation}
\label{eq:kde-split}
    \sup_{x \in \mathcal{M}} \bigl| \hat{p}(x) - \mu(x) \bigr|
    \;\le\; \sup_{x \in \mathcal{M}} \bigl| \hat{p}(x) - \mathbb{E}[\hat{p}(x)] \bigr|
    \;+\; \sup_{x \in \mathcal{M}} \bigl| \mathbb{E}[\hat{p}(x)] - \mu(x) \bigr|.
\end{equation}

\emph{Bias term.} Since $\mathbf{x}_j \overset{\mathrm{iid}}{\sim} \mu$, the expected
value of $\hat{p}(x)$ is given by
\begin{equation*}
    \mathbb{E}[\hat{p}(x)] \;=\; \frac{1}{(\pi h)^{d/2}} \int_{\mathcal{M}}
    \exp\!\left(-\frac{\|x - y\|^{2}}{h}\right) \mu(y)\, dV(y),
\end{equation*}
so applying Lemma~\ref{lem:bias} with $f = \mu$, there exist constants
$C_1 = C_1(\mathcal{M}, \mu) > 0$ and
$h_1(\mathcal{M},\mu) > 0$ such that for all $0 < h \le h_1$,
\begin{equation}
\label{eq:kde-bias}
    \sup_{x \in \mathcal{M}} \bigl| \mathbb{E}[\hat{p}(x)] - \mu(x) \bigr|
    \;\le\; C_1 h .
\end{equation}
Since this bound is deterministic, it holds with probability $1$.

\emph{Variance term.} Set $\eta := \sqrt{h}/2$, so that the kernel
$k_\eta(x,y) = \exp\bigl(-\|x-y\|^{2}/(4\eta^{2})\bigr)$ of
\cite[Section~3.1]{dww} coincides with $\exp\bigl(-\|x-y\|^{2}/h\bigr)$. By
\cite[Proposition~2, equation~(64)]{dww}, applied with bandwidth $\eta = \sqrt{h}/2$ to the
function class $\mathcal{K}_\eta = \{\, k_\eta(x,\cdot) : x \in \mathcal{M} \,\}$
of \cite[Definition~6]{dww}, there exist $h_2(\mathcal{M},\mu) > 0$ and
$C_{\mathrm{gc}} > 0$ such that for all $0 < h \le h_2(\mathcal{M},\mu)$, with
probability at least $1 - n^{-2}$,
\begin{equation}
\label{eq:kde-gc}
    \sup_{x \in \mathcal{M}} \left| \frac{1}{n}\sum_{j=1}^{n}
    \exp\!\left(-\frac{\|x - \mathbf{x}_j\|^{2}}{h}\right)
    - \int_{\mathcal{M}} \exp\!\left(-\frac{\|x - y\|^{2}}{h}\right)
    \mu(y)\, dV(y) \right|
    \;\le\; \frac{C_{\mathrm{gc}}}{\sqrt{n}}
    \left( \sqrt{\log(2/\sqrt{h})} + \sqrt{\log n} \right),
\end{equation}
where $C_{\mathrm{gc}}$ depends only on $d$, the diameter of $\mathcal{M}$, and
the uniform upper and lower bounds of $\mu$, all of which are finite since
$\mathcal{M}$ is compact and $\mu$ is positive and continuous. 

The condition
$h \ge 4 n^{-2}$ gives $\eta = \sqrt{h}/2 \ge 1/n$, hence
$\log(2/\sqrt{h}) = \log(1/\eta) \le \log n$, so that the right-hand side of
\eqref{eq:kde-gc} is at most $2 C_{\mathrm{gc}} \sqrt{\log n / n}$. Dividing
\eqref{eq:kde-gc} by $(\pi h)^{d/2}$, the left-hand side becomes
$\sup_{x \in \mathcal{M}} \bigl| \hat{p}(x) - \mathbb{E}[\hat{p}(x)] \bigr|$,
and we obtain
\begin{equation}
\label{eq:kde-var}
    \sup_{x \in \mathcal{M}} \bigl| \hat{p}(x) - \mathbb{E}[\hat{p}(x)] \bigr|
    \;\le\; C_3 \sqrt{\frac{\log n}{n h^{d}}},
    \qquad C_3 := 2 \pi^{-d/2} C_{\mathrm{gc}} .
\end{equation}

Substituting \eqref{eq:kde-bias} and \eqref{eq:kde-var} into
\eqref{eq:kde-split} and setting $h_0 := \min\{h_1, h_2\}$ and
$C := \max\{C_1, C_3\}$ yields \eqref{eq:kde-uniform}.
\end{proof}

\subsection{Proof of Lemma~\ref{lem:self-normalizing kde}}\label{sec:proof_density_estimated}

\begin{proof}
Let $h_0^{\star}$ and $C^{\star}$
be the constants in Lemma~\ref{lem:kde-uniform}, both depending only on
$\mathcal{M}_x$ and $\mu$. For a fixed bandwidth $h$, let $\hat p$ denote the
kernel density estimator in~\eqref{eq:kde1} constructed from all $m$
samples $\mathbf{x}_1, \dots, \mathbf{x}_m$ and for each $i\in [m]$, we define 
\begin{equation}
\label{eq:A1-loo-def}
    \hat p_{-i}(\mathbf{x}_i) := \frac{1}{m-1} \sum_{j \neq i}^m
    \frac{1}{(\pi h)^{d/2}} K_h(\mathbf{x}_i, \mathbf{x}_j) = \frac{\hat q(\mathbf{x}_i)}{(\pi h)^{d/2}}, \quad \bar Z = \frac{1}{m}\sum_{i=1}^{m} \hat p_{-i}(\mathbf{x}_i). 
\end{equation}
Note by definition of $\hat f(\mathbf{x}_i)$ in~\eqref{eq:self_KDE}, $\hat f(\mathbf{x}_i) = \frac{\hat p_{-i}(\mathbf{x}_i)}{\bar Z}$. We also note that since $\mu$ is positive and continuous on the compact
manifold $\mathcal{M}_x$, there exist $0 < c_1 \le c_2 < \infty$ with $
    c_1 \le \mu(\mathbf{x}) \le c_2, \forall\, \mathbf{x} \in \mathcal{M}_x$.

We take $m_0 \ge 2$ large enough that, for all $m \ge m_0$: (i) $(\kappa \log m / m)^{1/d} \le h_0$, so that \eqref{eq:A1-window} is
nonempty; and (ii) $(\kappa \log m / m)^{1/d} \ge 4 m^{-2}$, so that
\eqref{eq:A1-window} implies $4 m^{-2} \le h \le h_0 \le h_0^{\star}$ and
Lemma~\ref{lem:kde-uniform} applies with sample size $m$. Each of these holds
for all sufficiently large $m$. The constants $h_0 \le \min\{1, h_0^{\star}\}$
and $\kappa$ are fixed at the end of Step~3; the constants $C_1, C_2, C_3$
introduced below depend only on $C^{\star}$, $c_1$, $c_2$ and $d$, and in
particular not on $h_0$ or $\kappa$.

\emph{Step 1:} We first bound $\bigl| \hat p_{-i}(\mathbf{x}_i) - \mu(\mathbf{x}_i) \bigr|$.

Since $K_h(\mathbf{x}_i, \mathbf{x}_i) = 1$, we note $\hat p(\mathbf{x}_i)$ and $\hat p_{-i}(\mathbf{x}_i)$ are related by
\begin{equation}\label{eq:leave}
    \hat p_{-i}(\mathbf{x}_i) = \frac{m}{m-1}\, \hat p(\mathbf{x}_i)
    - \frac{1}{(m-1)(\pi h)^{d/2}} , \quad \forall i \in [m].
\end{equation}
Using~\eqref{eq:leave}, we obtain, for any $i \in [m]$:
\begin{equation}
\label{eq:A1-loo-bound}
\begin{aligned}
    &\bigl| \hat p_{-i}(\mathbf{x}_i) - \mu(\mathbf{x}_i) \bigr|
    = \left| \frac{m}{m-1}\bigl( \hat p(\mathbf{x}_i) - \mu(\mathbf{x}_i) \bigr)
       + \frac{\mu(\mathbf{x}_i)}{m-1}
       - \frac{1}{(m-1)\, \pi^{d/2}\, h^{d/2}} \right| \\
    &\le \frac{m}{m-1}
       \bigl| \hat p(\mathbf{x}_i) - \mu(\mathbf{x}_i) \bigr|
       + \frac{c_2}{m-1}
       + \frac{1}{(m-1)\, \pi^{d/2}\, h^{d/2}} 
    \le 2 \bigl| \hat p(\mathbf{x}_i) - \mu(\mathbf{x}_i) \bigr|
       + \frac{2 c_2}{m} + \frac{2}{\pi^{d/2}\, m\, h^{d/2}} \\
    & \le 2
       \bigl| \hat p(\mathbf{x}_i) - \mu(\mathbf{x}_i) \bigr| + (2 c_2 + 2\pi^{-d/2}) \sqrt{\frac{\log m}{m h^{d}}}.
\end{aligned}
\end{equation}
where the first line uses \eqref{eq:leave}, the second line follows from the triangle inequality and the bounds $m/(m-1) \le 2$ and $1/(m-1) \le 2/m$,
both valid for $m \ge m_0 \ge 2$, and the third line from
$\tfrac{1}{m h^{d/2}} \le \sqrt{\tfrac{\log m}{m h^{d}}}$ and
$\tfrac{1}{m} \le \sqrt{\tfrac{\log m}{m h^{d}}}$, valid for
$m \ge m_0 \ge 2$ and $h \le h_0 \le 1$.
    
By (ii) the bandwidth $h$ lies in the admissible range of
Lemma~\ref{lem:kde-uniform}, so with probability at least $1 - m^{-2}$,
\begin{equation}
\label{eq:A1-event1}
    \sup_{\mathbf{x} \in \mathcal{M}_x} \bigl| \hat p(\mathbf{x}) - \mu(\mathbf{x}) \bigr|
    \;\le\; C^{\star} \left( h + \sqrt{\frac{\log m}{m h^{d}}} \right).
\end{equation}
Call this event $\mathcal{A}$. Since \eqref{eq:A1-event1} is uniform over
$\mathcal{M}_x$, it holds simultaneously at every sample point $\mathbf{x}_i$. 
Combining~\eqref{eq:A1-loo-bound} and ~\eqref{eq:A1-event1}, we have: with probability at least $1 - m^{-2}$,
\begin{equation}
\label{eq:A1-step1}
    \max_{i \in [m]} \bigl| \hat p_{-i}(\mathbf{x}_i) - \mu(\mathbf{x}_i) \bigr|
    \;\le\; C_1 \left( h + \sqrt{\frac{\log m}{m h^{d}}} \right),
    \qquad C_1 := 2 C^{\star} + 2 c_2 + 2\pi^{-d/2} .
\end{equation}

\emph{Step 2:} we next bound $\bigl| \bar Z - (\|\mu\|_{L^2})^{2} \bigr|$.
\begin{equation}
\label{eq:A1-Zsplit}
    \bar Z - (\|\mu\|_{L^2})^{2}
    = \frac{1}{m}\sum_{i=1}^{m}
      \bigl( \hat p_{-i}(\mathbf{x}_i) - \mu(\mathbf{x}_i) \bigr)
    + \left( \frac{1}{m}\sum_{i=1}^{m} \mu(\mathbf{x}_i) - (\|\mu\|_{L^2})^{2} \right). 
\end{equation}
On event $\mathcal{A}$, the first term is bounded by
$C_1 ( h + \sqrt{\log m/(m h^{d})})$ by \eqref{eq:A1-step1}. 
For the second, since $\mathbb{E}[\mu(\mathbf{x})] = \int_{\mathcal{M}_x} \mu(\mathbf{x}) \cdot
\mu(\mathbf{x}) \, d\omega(\mathbf{x}) = (\|\mu\|_{L^2})^{2}$ and
$\mathbf{x}_1, \dots, \mathbf{x}_m$ are i.i.d. with
$\mu(\mathbf{x}_i) \in [c_1, c_2]$, Hoeffding's inequality gives, for any $t>0$,
\begin{equation}
\label{eq:A1-hoeffding}
    \mathbb{P}\left( \left| \frac{1}{m}\sum_{i=1}^{m} \mu(\mathbf{x}_i)
    - (\|\mu\|_{L^2})^{2} \right| \ge t \right)
    \;\le\; 2 \exp\!\left( -\frac{2 m t^{2}}{(c_2 - c_1)^{2}} \right).
\end{equation}
Choosing $t = (c_2 - c_1)\sqrt{\log(2m^{2})/(2m)}$ makes the right-hand side of
\eqref{eq:A1-hoeffding} equal to $m^{-2}$, and with probability at least $1 - m^{-2}$,
\begin{equation}
\label{eq:A1-event2}
    \left| \frac{1}{m}\sum_{i=1}^{m} \mu(\mathbf{x}_i) - (\|\mu\|_{L^2})^{2} \right|
    \;\le\; (c_2 - c_1) \sqrt{\frac{\log(2m^{2})}{(2m)}} \;\le\; 2 (c_2 - c_1) \sqrt{\frac{\log m}{m}}
    \;\le\; 2 (c_2 - c_1) \sqrt{\frac{\log m}{m h^{d}}},
\end{equation}
where we use $m \ge m_0 \ge 2$, and
$h \le h_0 \le 1$. Call this event $\mathcal{B}$ and define $C_2: = 2 (c_2 - c_1)$. 

By the union bound
$\mathbb{P}(\mathcal{A} \cap \mathcal{B}) \ge 1 - 2m^{-2}$, and on
$\mathcal{A} \cap \mathcal{B}$, combining
\eqref{eq:A1-Zsplit}--\eqref{eq:A1-event2},
\begin{equation}
\label{eq:A1-step2}
    \bigl| \bar Z - (\|\mu\|_{L^2})^{2} \bigr|
    \;\le\; C_3 \left( h + \sqrt{\frac{\log m}{m h^{d}}} \right),
    \qquad C_3 := C_1 + C_2 .
\end{equation}

\emph{Step 3: }We next bound $\bar Z$ away from zero. Under \eqref{eq:A1-window} we have
$h \le h_0$ and $m h^{d} \ge \kappa \log m$, hence
\begin{equation}
\label{eq:A1-Ebound}
    h + \sqrt{\frac{\log m}{m h^{d}}}
    \;\le\; h_0 + \frac{1}{\sqrt{\kappa}} .
\end{equation}
Choosing
\begin{equation}
\label{eq:A1-choice}
    h_0 \;:=\; \min\left\{ 1,\; h_0^{\star},\;
    \frac{(\|\mu\|_{L^2})^{2}}{4 C_3} \right\},
    \qquad
    \kappa \;:=\; \left( \frac{4 C_3}{(\|\mu\|_{L^2})^{2}} \right)^{2},
\end{equation}
both of which depend only on $\mathcal{M}_x$ and $\mu$, inequalities
\eqref{eq:A1-Ebound} and \eqref{eq:A1-step2} give
$| \bar Z - (\|\mu\|_{L^2})^{2} | \le (\|\mu\|_{L^2})^{2}/2$ and therefore, on $\mathcal{A} \cap \mathcal{B}$:
\begin{equation}
\label{eq:A1-Zlower}
    \bar Z \;\ge\; \frac{(\|\mu\|_{L^2})^{2}}{2}.
\end{equation}
 
\begin{equation*}
\begin{aligned}
    &\left| \hat f(\mathbf{x}_i)
    - \frac{\mu(\mathbf{x}_i)}{(\|\mu\|_{L^2})^{2}} \right|
    = \left| \frac{\hat p_{-i}(\mathbf{x}_i)}{\bar Z}
       - \frac{\mu(\mathbf{x}_i)}{(\|\mu\|_{L^2})^{2}} \right| 
    = \left| \frac{(\|\mu\|_{L^2})^{2}
       \bigl( \hat p_{-i}(\mathbf{x}_i) - \mu(\mathbf{x}_i) \bigr)
       - \mu(\mathbf{x}_i)
       \bigl( \bar Z - (\|\mu\|_{L^2})^{2} \bigr)}
       {(\|\mu\|_{L^2})^{2}\, \bar Z} \right| \\
    &\le \frac{(\|\mu\|_{L^2})^{2}
       \bigl| \hat p_{-i}(\mathbf{x}_i) - \mu(\mathbf{x}_i) \bigr|
       + \mu(\mathbf{x}_i)
       \bigl| \bar Z - (\|\mu\|_{L^2})^{2} \bigr|}
       {(\|\mu\|_{L^2})^{2}\, \bar Z} 
    \le \frac{(\|\mu\|_{L^2})^{2} C_1
       \left( h + \sqrt{\dfrac{\log m}{m h^{d}}} \right)
       + c_2 C_3
       \left( h + \sqrt{\dfrac{\log m}{m h^{d}}} \right)}
       {(\|\mu\|_{L^2})^{2}\, \bar Z} \\
    &\le \frac{2 \bigl[ (\|\mu\|_{L^2})^{2} C_1 + c_2 C_3 \bigr]}
       {(\|\mu\|_{L^2})^{4}}
       \left( h + \sqrt{\frac{\log m}{m h^{d}}} \right),
\end{aligned}
\end{equation*}
where we use triangle inequality,~\eqref{eq:A1-step1} and~\eqref{eq:A1-step2} in the second line, and~\eqref{eq:A1-Zlower} in the last line. This concludes the proof.
\end{proof}

\subsection{Proof of Theorem~\ref{thm:finite_population_density_estimated}}\label{sec:proof density estimated}

\begin{proof}
    Let $\Alpha^{(\theta)}$ and $\Beta^{(\theta)}$ be the scaling factors in Algorithm~\ref{alg:EOT_invariant} that diagonally scale $\mathbf{M}^{(\theta)}$, defined in~\eqref{eq:M}, to satisfy the row and column constraints in~\eqref{eq:row_col_constraint}, for fixed $\theta \in [0,1]$, where $\hat{\mathbf{f}}$ and $\hat{\mathbf{g}}$ are the mean-normalized density estimates in~\eqref{eq:self_KDE} applied to $\mathbf{X}$ and $\mathbf{Y}$ with bandwidths $h_1$ and $h_2$. Define approximate scaling factors $\hat{\Alpha}^{(\theta)}$ and $\hat{\Beta}^{(\theta)}$ as
    \begin{equation}\label{eq:hat_scaling_factor_est}
        \hat{\alpha}_i^{(\theta)} = \frac{u^{(\theta)}_\varepsilon(\mathbf{x}_i)}{(\|\mu\|_{L^2})^{\theta}\sqrt{\Zx}},\quad \text{ and } \quad \hat{\beta}_j^{(\theta)} = \frac{v^{(\theta)}_\varepsilon(\mathbf{y}_j)}{(\|\nu\|_{L^2})^{\theta}\sqrt{\Zy}},
    \end{equation}
    for any $i \in [m]$ and any $j \in [n]$, where $u^{(\theta)}_\varepsilon$ and $v^{(\theta)}_\varepsilon$ are defined in~\eqref{eq:integral_eq}. For brevity we write
    \begin{equation}\label{eq:rate_def}
        \rho_1 := h_1 + \sqrt{\frac{\log m}{m h_1^{d_1}}},
        \qquad
        \rho_2 := h_2 + \sqrt{\frac{\log n}{n h_2^{d_2}}}.
    \end{equation}

    \emph{The density-estimation event.} Let $\mathcal{E}$ denote the event on which the conclusion of Lemma~\ref{lem:self-normalizing kde}, applied to $\mathbf{X}$ with bandwidth $h_1$ and to $\mathbf{Y}$ with bandwidth $h_2$, holds simultaneously:
    \begin{equation}\label{eq:kde_event}
        \left| \hat f(\mathbf{x}_i) - \frac{\mu(\mathbf{x}_i)}{(\|\mu\|_{L^2})^{2}} \right| \le C_0^x \rho_1
        \quad \text{and} \quad
        \left| \hat g(\mathbf{y}_j) - \frac{\nu(\mathbf{y}_j)}{(\|\nu\|_{L^2})^{2}} \right| \le C_0^y \rho_2,
        \qquad \forall\, i \in [m],\ \forall\, j \in [n],
    \end{equation}
    where $C_0^x = C_0(\mathcal{M}_x,\mu)$ and $C_0^y = C_0(\mathcal{M}_y,\nu)$ are defined in Lemma~\ref{lem:self-normalizing kde}. Applying the union bound over each dataset and under Assumption~\ref{assump:m_and_n} ($m \ge n^{\gamma}$ and $\gamma \le 1$), we have:
    \begin{equation}\label{eq:kde_event_prob}
        \mathbb{P}(\mathcal{E}) \;\ge\; 1 - 2m^{-2} - 2n^{-2} \;\ge\; 1 - 4 n^{-2\gamma}.
    \end{equation} 

    We next show that on $\mathcal{E}$,  $\hat f(\mathbf{x})$ and $\hat g(\mathbf{y})$ are bounded away from zero and infinity for any $\mathbf{x}$ and $\mathbf{y}$.

    Since $\mu$ and $\nu$ are positive and continuous on the compact manifolds
    $\mathcal{M}_x$ and $\mathcal{M}_y$, there exist constants
    $0 < c_1 \le c_2 < \infty$ and $0 < c_3 \le c_4 < \infty$ with
    \begin{equation}\label{eq:density_bounds}
        c_1 \le \mu(\mathbf{x}) \le c_2 \quad \forall\, \mathbf{x} \in \mathcal{M}_x,
        \qquad
        c_3 \le \nu(\mathbf{y}) \le c_4 \quad \forall\, \mathbf{y} \in \mathcal{M}_y .
    \end{equation}
    Writing $B_\mu := (\|\mu\|_{L^2})^{2}$ and $B_\nu := (\|\nu\|_{L^2})^{2}$, we have $\mu(\mathbf{x}_i)/B_\mu \in [c_1/B_\mu, c_2/B_\mu]$ and
    $\nu(\mathbf{y}_j)/B_\nu \in [c_3/B_\nu, c_4/B_\nu]$.
    Under~\eqref{eq:bandwidth_window}, we have 
    $\rho_1 \le h_1' + \kappa_1^{-1/2}$, and $\rho_2 \le h_2' + \kappa_2^{-1/2}$. Choosing
    \begin{equation}\label{eq:h_kappa_choice}
        h_1' \le \frac{c_1}{4 C_0^x B_\mu}, \quad
        \kappa_1 \ge \left( \frac{4 C_0^x B_\mu}{c_1} \right)^{2},
        \qquad
        h_2' \le \frac{c_3}{4 C_0^y B_\nu}, \quad
        \kappa_2 \ge \left( \frac{4 C_0^y B_\nu}{c_3} \right)^{2},
    \end{equation}
    all of which depend only on $\mathcal{M}_x, \mathcal{M}_y, \mu$ and $\nu$,
    gives $C_0^x \rho_1 \le c_1/(2 B_\mu)$ and
    $C_0^y \rho_2 \le c_3/(2 B_\nu)$. Combining with~\eqref{eq:kde_event}, on
    $\mathcal{E}$ we obtain
    \begin{equation}\label{eq:fhat_interval}
        \hat f(\mathbf{x}_i) \in I_x := \left[ \frac{c_1}{2 B_\mu},\, \frac{2 c_2}{B_\mu} \right],
        \qquad
        \hat g(\mathbf{y}_j) \in I_y := \left[ \frac{c_3}{2 B_\nu},\, \frac{2 c_4}{B_\nu} \right],
        \qquad \forall\, i \in [m],\ \forall\, j \in [n],
    \end{equation}
    and the intervals $I_x, I_y$ also contain the targets
    $\mu(\mathbf{x}_i)/B_\mu$ and $\nu(\mathbf{y}_j)/B_\nu$. In particular
    $I_x$ and $I_y$ are fixed compact subintervals of $(0,\infty)$, independent
    of $m$, $n$, $h_1$ and $h_2$.

    The map $\phi(z) = z^{-\theta}$ is continuously differentiable on
    $(0,\infty)$ with $\phi'(z) = -\theta z^{-\theta-1}$. For $\theta \in [0,1]$, using~\eqref{eq:fhat_interval}, we have:
    \begin{equation}\label{eq:lipschitz}
        L_x := \sup_{z \in I_x} |\phi'(z)| \le \left( \frac{2 B_\mu}{c_1} \right)^{2},
        \qquad
        L_y := \sup_{z \in I_y} |\phi'(z)| \le \left( \frac{2 B_\nu}{c_3} \right)^{2} .
    \end{equation}

    Applying the mean value theorem, we have: on $\mathcal{E}$, for any $i \in [m]$,
    \begin{equation}\label{eq:f_inv_est}
    \begin{aligned}
        \left| \frac{1}{(\hat f(\mathbf{x}_i))^{\theta}}
        - \frac{(\|\mu\|_{L^2})^{2\theta}}{(\mu(\mathbf{x}_i))^{\theta}} \right|
        &= \left| \phi\bigl( \hat f(\mathbf{x}_i) \bigr)
           - \phi\!\left( \frac{\mu(\mathbf{x}_i)}{B_\mu} \right) \right|
         = \bigl| \phi'(\xi_i) \bigr| \cdot
           \left| \hat f(\mathbf{x}_i) - \frac{\mu(\mathbf{x}_i)}{B_\mu} \right| 
        \le  L_x C_0^x \rho_1,
    \end{aligned}
    \end{equation}
    where $\xi_i$ lies between $\hat f(\mathbf{x}_i)$ and
    $\mu(\mathbf{x}_i)/B_\mu$, hence in $I_x$ by~\eqref{eq:fhat_interval}. Equivalently, we have:
    \begin{equation}\label{eq:1/f}
        \frac{1}{(\hat f(\mathbf{x}_i))^{\theta}} = \frac{(\|\mu\|_{L^2})^{2\theta}}{(\mu(\mathbf{x}_i))^{\theta}} + \mathcal{O}(\rho_1).
    \end{equation}
    An analogous derivation
    gives, for any $j \in [n]$,
    \begin{equation}\label{eq:1/g}
        \frac{1}{(\hat g(\mathbf{y}_j))^{\theta}} = \frac{(\|\nu\|_{L^2})^{2\theta}}{(\nu(\mathbf{y}_j))^{\theta}} + \mathcal{O}(\rho_2).
    \end{equation}
    Both bounds hold simultaneously for all $i \in [m]$ and $j \in [n]$, with
    absorbed constants depending only on $\mathcal{M}_x, \mathcal{M}_y, \mu, \nu$ and
    $\theta$.

    \emph{Bounding $S^{(\theta)}$.} Applying~\eqref{eq:1/f} and then Hoeffding's inequality, we have on $\mathcal{E}$
    \begin{equation}\label{eq:S1_est}
    \begin{aligned}
        \frac{1}{m} \sum_{i=1}^m \frac{1}{(\hat f(\mathbf{x}_i))^{\theta}}
        &= \frac{1}{m} \sum_{i=1}^m \frac{(\|\mu\|_{L^2})^{2\theta}}{(\mu(\mathbf{x}_i))^{\theta}} + \mathcal{O}(\rho_1) 
        = (\|\mu\|_{L^2})^{2\theta} \int_{\mathcal{M}_x} \frac{1}{(\mu(\mathbf{x}))^{\theta}} \mu(\mathbf{x})\, d\omega(\mathbf{x}) + \widetilde{\mathcal{O}}_{m}\left(\sqrt{\frac{\log m}{m}}\right) + \mathcal{O}(\rho_1)
        \; \\
        &=\; (\|\mu\|_{L^2})^{2\theta}\, \Zx + \widetilde{\mathcal{O}}_{m}(\rho_1),
    \end{aligned}
    \end{equation}
    where in the last step we use $\sqrt{\log m/m} \le \sqrt{\log m/(m h_1^{d_1})} \le \rho_1$,
    valid since $h_1 \le h_1' \le 1$, and absorb the term $\mathcal{O}(\rho_1)$,
    a deterministic bound on $\mathcal{E}$, into
    $\widetilde{\mathcal{O}}_{m}(\rho_1)$.
    Similarly,
    \begin{equation}\label{eq:S2_est}
        \frac{1}{n} \sum_{j=1}^n \frac{1}{(\hat g(\mathbf{y}_j))^{\theta}}
        = (\|\nu\|_{L^2})^{2\theta}\, \Zy + \widetilde{\mathcal{O}}_{n}(\rho_2).
    \end{equation}

Combining~\eqref{eq:S1_est} and~\eqref{eq:S2_est}, we have
    \begin{equation}\label{eq:S_bound_est}
    \begin{aligned}
        &\left| S^{(\theta)} - \frac{(\|\mu\|_{L^2})^{2\theta}\Zx}{(\|\nu\|_{L^2})^{2\theta}\Zy} \right|
        = \left| \frac{\frac{1}{m}\sum_{i=1}^m (\hat f(\mathbf{x}_i))^{-\theta}}
                      {\frac{1}{n}\sum_{j=1}^n (\hat g(\mathbf{y}_j))^{-\theta}}
          - \frac{(\|\mu\|_{L^2})^{2\theta}\Zx}{(\|\nu\|_{L^2})^{2\theta}\Zy} \right| \\
        &= \left| \frac{(\|\nu\|_{L^2})^{2\theta}\Zy
                 \left( \frac{1}{m}\sum_{i=1}^m (\hat f(\mathbf{x}_i))^{-\theta} - (\|\mu\|_{L^2})^{2\theta}\Zx \right)
                 - (\|\mu\|_{L^2})^{2\theta}\Zx
                 \left( \frac{1}{n}\sum_{j=1}^n (\hat g(\mathbf{y}_j))^{-\theta} - (\|\nu\|_{L^2})^{2\theta}\Zy \right)}
                 {(\|\nu\|_{L^2})^{2\theta}\Zy \cdot \frac{1}{n}\sum_{j=1}^n (\hat g(\mathbf{y}_j))^{-\theta}} \right| \\
        &\le \frac{(\|\nu\|_{L^2})^{2\theta}\Zy
                 \left| \frac{1}{m}\sum_{i=1}^m (\hat f(\mathbf{x}_i))^{-\theta} - (\|\mu\|_{L^2})^{2\theta}\Zx \right|
                 + (\|\mu\|_{L^2})^{2\theta}\Zx
                 \left| \frac{1}{n}\sum_{j=1}^n (\hat g(\mathbf{y}_j))^{-\theta} - (\|\nu\|_{L^2})^{2\theta}\Zy \right|}
                 {(\|\nu\|_{L^2})^{2\theta}\Zy \cdot \frac{1}{n}\sum_{j=1}^n (\hat g(\mathbf{y}_j))^{-\theta}} \\
        &= \widetilde{\mathcal{O}}_{m,n}(\rho_1 + \rho_2),
    \end{aligned}
    \end{equation}
    where the
    third line uses the triangle inequality, and the last uses~\eqref{eq:S1_est}
    and~\eqref{eq:S2_est} together with the fact that on $\mathcal{E}$ the average
    $\frac{1}{n}\sum_{j=1}^n (\hat g(\mathbf{y}_j))^{-\theta}$ is bounded away
    from zero by a constant depending only on $\mathcal{M}_y$ and $\nu$,
    by~\eqref{eq:fhat_interval} and~\eqref{eq:1/g}. An analogous derivation gives
    \begin{equation}\label{eq:S_inv_bound_est}
        \left| \frac{1}{S^{(\theta)}}
        - \frac{(\|\nu\|_{L^2})^{2\theta}\Zy}{(\|\mu\|_{L^2})^{2\theta}\Zx} \right|
        \;=\; \widetilde{\mathcal{O}}_{m,n}(\rho_1 + \rho_2).
    \end{equation}

    \emph{Approximate scaling.} We next show that $\hat{\Alpha}^{(\theta)}$ and $\hat{\Beta}^{(\theta)}$ approximately scale $\mathbf{M}^{(\theta)}$ to have prescribed row sums $\mathbf{r}$ and column sums $\mathbf{c}$. Recall that, from~\eqref{eq:row_col_constraint},
    \begin{equation}\label{eq:row_col_sum_est}
        r_i = \frac{n}{(\hat f(\mathbf{x}_i))^{\theta}\sqrt{S^{(\theta)}}}, \quad \text{and} \quad
        c_j = \frac{m\sqrt{S^{(\theta)}}}{(\hat g(\mathbf{y}_j))^{\theta}}, \quad \forall i \in [m], \forall j \in [n].
    \end{equation}
    Using~\eqref{eq:hat_scaling_factor_est},~\eqref{eq:row_col_sum_est}, $\mathcal{K}_{\varepsilon}$ and $\mathbf{M}^{(\theta)}$ defined in~\eqref{eq:W_kernel} and~\eqref{eq:M}, we have on $\mathcal{E}$: for any $i \in [m]$,
    \begin{equation}\label{eq:row_bound_est}
    \begin{aligned}
        &\frac{1}{r_i} \sum_{j=1}^n \hat{\alpha}_i^{(\theta)} M_{ij}^{(\theta)} \hat{\beta}_j^{(\theta)}
        = \frac{(\hat f(\mathbf{x}_i))^{\theta} \sqrt{S^{(\theta)}}}{n} \sum_{j=1}^n \frac{u^{(\theta)}_\varepsilon(\mathbf{x}_i)}{(\|\mu\|_{L^2})^{\theta}\sqrt{\Zx}} \frac{\mathcal{K}_\varepsilon(\mathbf{x}_i,\mathbf{y}_j)}{(\hat f(\mathbf{x}_i)\, \hat g(\mathbf{y}_j))^{\theta}} \frac{v^{(\theta)}_\varepsilon(\mathbf{y}_j)}{(\|\nu\|_{L^2})^{\theta}\sqrt{\Zy}} \\
        &= \frac{\sqrt{S^{(\theta)}}}{(\|\mu\|_{L^2}\|\nu\|_{L^2})^{\theta}\sqrt{\Zx\Zy}} \cdot \frac{1}{n} \sum_{j=1}^n u^{(\theta)}_\varepsilon(\mathbf{x}_i) \mathcal{K}_\varepsilon(\mathbf{x}_i,\mathbf{y}_j) v^{(\theta)}_\varepsilon(\mathbf{y}_j) \frac{1}{(\hat g(\mathbf{y}_j))^{\theta}} \\
        &= \frac{(\|\nu\|_{L^2})^{\theta}\sqrt{S^{(\theta)}}}{(\|\mu\|_{L^2})^{\theta}\sqrt{\Zx\Zy}} \cdot \frac{1}{n} \sum_{j=1}^n u^{(\theta)}_\varepsilon(\mathbf{x}_i) \mathcal{K}_\varepsilon(\mathbf{x}_i,\mathbf{y}_j) v^{(\theta)}_\varepsilon(\mathbf{y}_j) \frac{1}{(\nu(\mathbf{y}_j))^{\theta}} + \mathcal{O}(\rho_2), \\
        &= \frac{(\|\nu\|_{L^2})^{\theta}\sqrt{S^{(\theta)}}}{(\|\mu\|_{L^2})^{\theta}\sqrt{\Zx\Zy}} \left( \int_{\mathcal{M}_y} u^{(\theta)}_\varepsilon(\mathbf{x}_i) \mathcal{K}_\varepsilon(\mathbf{x}_i,\mathbf{y}) v^{(\theta)}_\varepsilon(\mathbf{y}) (\nu(\mathbf{y}))^{1-\theta}\, d\eta(\mathbf{y}) + \widetilde{\mathcal{O}}_{n}\left(\sqrt{\frac{\log n}{n}}\right) \right) + \mathcal{O}(\rho_2) \\
        &= \frac{(\|\nu\|_{L^2})^{\theta}\sqrt{S^{(\theta)}}\, \Zy}{(\|\mu\|_{L^2})^{\theta}\sqrt{\Zx\Zy}} + \widetilde{\mathcal{O}}_{n}(\rho_2)
        \;=\; 1 + \widetilde{\mathcal{O}}_{m,n}(\rho_1 + \rho_2) 
    \end{aligned}
    \end{equation}
    where the third line replaces $(\hat g(\mathbf{y}_j))^{-\theta}$ using~\eqref{eq:1/g}, together with the fact that $u^{(\theta)}_\varepsilon$, $v^{(\theta)}_\varepsilon$ and $\mathcal{K}_\varepsilon$ are bounded on the compact manifolds, the fourth line applies Hoeffding's inequality and used~\eqref{eq:integral_eq}, and the last line used~\eqref{eq:S_bound_est}.
    Under Assumption~\ref{assump:m_and_n}, a union bound over all $i \in [m]$ shows that~\eqref{eq:row_bound_est} holds for all $i \in [m]$. 
    
    Following an analogous derivation, for all $j \in [n]$,
    \begin{equation}\label{eq:col_bound_est}
        \frac{1}{c_j} \sum_{i=1}^m \hat{\alpha}_i^{(\theta)} M_{ij}^{(\theta)} \hat{\beta}_j^{(\theta)}
        = 1 + \widetilde{\mathcal{O}}_{m,n}(\rho_1+\rho_2).
    \end{equation}

    \emph{Applying the matrix scaling stability lemma.} We now apply Lemma~\ref{lem:scaling_factor_concen} to $\hat{\Alpha}^{(\theta)}$ and $\hat{\Beta}^{(\theta)}$, and first bound the quantities appearing there. By the definition of $s$ and using~\eqref{eq:S1_est}, and~\eqref{eq:S_inv_bound_est}, 
    \begin{equation}\label{eq:s_est}
        s = \|\mathbf{r}\|_1
        = \frac{nm}{\sqrt{S^{(\theta)}}} \cdot \frac{1}{m}\sum_{i=1}^m \frac{1}{(\hat f(\mathbf{x}_i))^{\theta}}
        = nm \left( (\|\mu\|_{L^2}\|\nu\|_{L^2})^{\theta}\sqrt{\Zx\Zy} + \widetilde{\mathcal{O}}_{m,n}(\rho_1 + \rho_2)  \right).
    \end{equation}
    Since $u^{(\theta)}_\varepsilon$ and $v^{(\theta)}_\varepsilon$ are positive and continuous on the compact manifolds, and since $\hat f$ and $\hat g$ are bounded away from zero and infinity on $\mathcal{E}$, there exist positive constants $C^{'}_{1}(\varepsilon)$ and $C^{'}_{2}(\varepsilon)$ such that $C_1$ and $C_2$ defined in Lemma~\ref{lem:scaling_factor_concen} satisfy
    \begin{equation}\label{eq:lemma_C_est}
        C_1 = \min_i \left\{\frac{\hat{\alpha}_i^{(\theta)}}{r_i}\right\} \geq \frac{C^{'}_{1}(\varepsilon)}{n}, \quad
        C_2 = \min_j \left\{\frac{\hat{\beta}_j^{(\theta)}}{c_j}\right\} \geq \frac{C^{'}_{2}(\varepsilon)}{m}.
    \end{equation}
    Similarly, there exist positive constants $C^{'}_{3}(\varepsilon)$ and $C^{'}_{4}(\varepsilon)$ such that
    \begin{equation}\label{eq:lemma_M_est}
        C^{'}_{3}(\varepsilon) \leq M_{ij}^{(\theta)} = \frac{\operatorname{exp}\left(-\|\mathbf{x}_i - \mathbf{y}_j\|_2^2/\varepsilon\right)}{(\hat f(\mathbf{x}_i)\, \hat g(\mathbf{y}_j))^{\theta}} \leq C^{'}_{4}(\varepsilon),
    \end{equation}
    so that $a$ and $b$ defined in Lemma~\ref{lem:scaling_factor_concen} satisfy $C^{'}_{3}(\varepsilon) \leq a \leq b \leq C^{'}_{4}(\varepsilon)$. Hence, for any $\delta \in (0,1/2]$, using~\eqref{eq:s_est},~\eqref{eq:lemma_C_est}, and~\eqref{eq:lemma_M_est}, the coefficients of~\eqref{eq:lemma_alpha_bound} and~\eqref{eq:lemma_beta_bound} are bounded by positive constants $D_1(\varepsilon)$ and $D_2(\varepsilon)$, exactly as in~\eqref{eq:constant}.

    Applying Lemma~\ref{lem:scaling_factor_concen} with $\delta := \tau\, C(\varepsilon)(\rho_1 + \rho_2)$, where $C(\varepsilon)$ is the constant associated with~\eqref{eq:row_bound_est} and~\eqref{eq:col_bound_est} in Definition~\ref{def:order_prob} and $\delta \le 1/2$ upon choosing $n_0(\varepsilon,\tau)$, $h_1'$ and $h_2'$ properly, there exists a pair of positive vectors $(\Alpha^{(\theta)'}, \Beta^{(\theta)'})$ that scales $\mathbf{M}^{(\theta)}$ to row sums $\mathbf{r}$ and column sums $\mathbf{c}$, such that
    \begin{equation}\label{eq:scaling_concen_est}
        \frac{|\alpha_i^{(\theta)'} - \hat{\alpha}_i^{(\theta)}|}{\hat{\alpha}_i^{(\theta)}}
        = \widetilde{\mathcal{O}}_{m,n}^{(\varepsilon)}(\rho_1 + \rho_2),
        \quad \text{ and } \quad
        \frac{|\beta_j^{(\theta)'} - \hat{\beta}_j^{(\theta)}|}{\hat{\beta}_j^{(\theta)}}
        = \widetilde{\mathcal{O}}_{m,n}^{(\varepsilon)}(\rho_1 + \rho_2) ,
    \end{equation}
    for any $i \in [m]$ and any $j \in [n]$.

    \emph{Conclusion.} By Lemma~\ref{lemma:existence_uniqueness_scaling} and the definition of scaling factors, the two pairs $(\Alpha^{(\theta)}, \Beta^{(\theta)})$ and $(\Alpha^{(\theta)'}, \Beta^{(\theta)'})$ satisfy $\alpha_i^{(\theta)}\beta_j^{(\theta)} = \alpha_i^{(\theta)'}\beta_j^{(\theta)'}$ for any $i \in [m]$ and $j \in [n]$. Hence, for any $i \in [m]$ and $j \in [n]$,
    \begin{equation}\label{eq:mul1_est}
    \begin{aligned}
        W_{ij}^{(\theta)} &= \alpha_i^{(\theta)} e^{-\|\mathbf{x}_i - \mathbf{y}_j\|_2^2/\varepsilon} \beta_j^{(\theta)}
        = \alpha_i^{(\theta)'} e^{-\|\mathbf{x}_i - \mathbf{y}_j\|_2^2/\varepsilon} \beta_j^{(\theta)'} 
        = \hat{\alpha}_i^{(\theta)} e^{-\|\mathbf{x}_i - \mathbf{y}_j\|_2^2/\varepsilon} \hat{\beta}_j^{(\theta)} \left( 1 + \widetilde{\mathcal{O}}_{m,n}^{(\varepsilon)}(\rho_1 + \rho_2) \right) \\
        &= \frac{\mathcal{W}^{(\theta)}_{\varepsilon}(\mathbf{x}_i,\mathbf{y}_j)}{(\|\mu\|_{L^2}\|\nu\|_{L^2})^{\theta}\sqrt{\Zx\Zy}} \left( 1 + \widetilde{\mathcal{O}}_{m,n}^{(\varepsilon)}(\rho_1 + \rho_2) \right).
    \end{aligned}
    \end{equation}
    Since $\mathcal{W}_\varepsilon^{(\theta)}$ is continuous and bounded above on the compact product manifold $\mathcal{M}_x \times \mathcal{M}_y$, the multiplicative estimate in~\eqref{eq:mul1_est} implies the additive bound in~\eqref{eq:W_ij2W_kernel_est}. Finally, this estimate holds on the intersection of $\mathcal{E}$ with the events underlying~\eqref{eq:S1_est}--\eqref{eq:scaling_concen_est}; by~\eqref{eq:kde_event_prob} and the union bound, this intersection has probability at least $1 - n^{-\tau} - 4n^{-2\gamma}$, completing the proof.
\end{proof}
\end{appendices}
  
\bibliographystyle{unsrt}
\bibliography{references}
\end{document}